\ifdefined\pdfminorversion\pdfminorversion=7\fi
\documentclass[11pt,onecolumn,copyright,gdm]{google}
\usepackage[authoryear,sort&compress,round]{natbib}
\usepackage{subcaption}
\usepackage{multirow}
\usepackage{algorithm}
\usepackage{algpseudocode}
\usepackage{etoolbox}
\usepackage{placeins}
\usepackage{wrapfig}
\usepackage{arydshln}
\usepackage{amsmath,amsfonts,bm}

\def\eqref#1{equation~\ref{#1}}
\def\1{\bm{1}}

\DeclareMathAlphabet{\mathsfit}{\encodingdefault}{\sfdefault}{m}{sl}
\SetMathAlphabet{\mathsfit}{bold}{\encodingdefault}{\sfdefault}{bx}{n}

\definecolor{routerbg}{RGB}{228,238,250}

\newcommand{\vhop}{\textsc{VHop}}
\newcommand{\router}{\textsc{VHop-Router}}

\newtheorem{theorem}{Theorem}

\newtheorem{proposition}{Proposition}

\newcommand{\dvrule}{\hspace{2.5pt}%
  \smash{\raisebox{-2.2pt}{\rule{0.55pt}{7.2pt}}}%
  \hspace{2.5pt}}

\newsavebox{\gdmtablebox}
\let\gdmoriginalresizebox\resizebox
\RenewDocumentCommand{\resizebox}{s m m +m}{%
  \IfBooleanTF{#1}{\gdmoriginalresizebox*{#2}{#3}{#4}}{%
    \ifstrequal{#3}{!}{%
      \sbox{\gdmtablebox}{#4}%
      \ifdim\wd\gdmtablebox>\dimexpr#2\relax
        \gdmoriginalresizebox{#2}{!}{\usebox{\gdmtablebox}}%
      \else
        \usebox{\gdmtablebox}%
      \fi
    }{\gdmoriginalresizebox{#2}{#3}{#4}}%
  }}
\makeatletter
\AtBeginEnvironment{tabular}{%
  \ifdefined\benchdash
    \renewcommand{\benchdash}{\hdashline[2pt/2pt]}%
  \fi}
\makeatother

\let\gdmoriginalvspace\vspace
\NewDocumentCommand{\gdmfloatvspace}{s m}{%
  \ifdim\dimexpr#2\relax<0pt\relax
  \else
    \IfBooleanTF{#1}{\gdmoriginalvspace*{#2}}{\gdmoriginalvspace{#2}}%
  \fi}
\AtBeginEnvironment{figure}{\let\vspace\gdmfloatvspace}
\AtBeginEnvironment{figure*}{\let\vspace\gdmfloatvspace}
\AtBeginEnvironment{table}{\let\vspace\gdmfloatvspace}
\AtBeginEnvironment{table*}{\let\vspace\gdmfloatvspace}
\let\gdmoriginalendwrapfigure\endwrapfigure
\newif\ifgdmstationarywrap
\NewDocumentCommand{\gdmstartwrap}{m m m m m}{%
  \IfValueTF{#2}{\def\gdmwraplines{[#2]}}{\def\gdmwraplines{}}%
  \ifgdmstationarywrap
    \def\gdmwrapside{#3}%
  \else
    \ifstrequal{#3}{l}{\def\gdmwrapside{L}}{\def\gdmwrapside{R}}%
  \fi
  \edef\gdmwrapstart{\expandafter\noexpand\csname gdmoriginalwrap#1\endcsname
    \gdmwraplines{\gdmwrapside}}%
  \IfValueTF{#4}{\gdmwrapstart[#4]{#5}}{\gdmwrapstart{#5}}}
\RenewDocumentCommand{\wraptable}{o m o m}{\gdmstartwrap{table}{#1}{#2}{#3}{#4}}
\makeatletter
\RenewDocumentEnvironment{wrapfigure}{o m o m +b}{%
  \in@{\vspace}{#5}%
  \ifin@\gdmstationarywraptrue\fi
  \gdmstartwrap{figure}{#1}{#2}{#3}{#4}%
  #5%
  \gdmoriginalendwrapfigure}{}
\makeatother
\makeatletter
\AtBeginDocument{%
  \patchcmd{\ttl@select}{\begingroup}{\begingroup\let\WF@floathand\relax}{}%
    {\PackageError{gdm_main}{Could not adapt wrapped headings}{Check titlesec compatibility.}}}
\makeatother
\newcommand{\vhopneedspace}[1]{\Needspace{#1}}

\AddToHook{file/appendix_cases.tex/before}{%
  \clearpage
  \begingroup
  \small
  \captionsetup{font=footnotesize,skip=6pt}%
  \setlength{\intextsep}{4pt plus 1pt minus 1pt}%
  \let\gdmcaseclearpage\clearpage
  \renewcommand{\clearpage}{%
    \par\begingroup\let\clearpage\gdmcaseclearpage\FloatBarrier\endgroup}}
\AddToHook{file/appendix_cases.tex/after}{\par\endgroup}

\titlespacing*{\paragraph}{0pt}{1.25ex plus 2pt minus 1pt}{0.6em}
\renewcommand{\bibfont}{\small}
\title{Learning to Route in Visual Space via Multi-Step Embedding Retrieval}
\author[1]{Tianyu Chen}
\author[1]{Mingyuan Zhou}
\author[2]{Jiaxing Wu}
\affil[1]{The University of Texas at Austin}
\affil[2]{\thepa{}}


\uselogo{}
\let\gdmoriginalcopyright\copyrightext
\correspondingauthor{jxwu@google.com\\
This work was done during Tianyu Chen's internship at Google DeepMind.}
\renewcommand{\copyrightext}{%
  \parbox[t]{\textwidth}{\footerfont\raggedright
    \gdmoriginalcopyright}}
\renewcommand{\today}{2026-09-27}
\hypersetup{
  pdftitle={Learning to Route in Visual Space via Multi-Step Embedding Retrieval},
  pdfauthor={Tianyu Chen, Mingyuan Zhou, and Jiaxing Wu},
  pdfsubject={Multi-hop visual retrieval},
  bookmarksnumbered=true,
  linkcolor=blue!60!black,
  citecolor=blue!60!black,
  urlcolor=blue!60!black
}

\begin{abstract}
LLM agents rely on retrieval tools to access external knowledge, yet visual agentic search remains severely bottlenecked by standard single-step retrievers. In current pipelines, the agent must issue text queries for every intermediate step, struggling when visual clues are difficult to describe or when the retriever fails to surface necessary intermediate evidence within its top results. We hypothesize that offloading multi-step navigation across the entire embedding space directly to the retrieval tool resolves this performance bottleneck. To study this systematically, we introduce \vhop{}, a flexible data generation framework and benchmark with five core difficulty levels testing both visual matching and search planning. Using this framework, we develop \router{}, an end-to-end training pipeline---combining supervised fine-tuning, online imitation learning, and reinforcement learning---that transforms a standard embedding model into an autoregressive multi-step retriever. Operating directly in the visual latent space, \router{} retrieves linked image chains in a single tool call without requiring the agent to formulate intermediate text queries. Experiments show \router{} boosts retrieval performance from under 5\% to 76.3\%. In agentic search, it improves task success rates by 52.7\% and reduces the average token length by 61\% from 1886 to 728, whereas upgrading the agent yields only a 3.7\% gain. Compared to a strong baseline where the agent retrieves the top 50 results per step, \router{} maintains superior performance while reducing in-context images by $23\times$ and cutting the cumulative API payload by $35\times$. The models also generalize robustly to unseen difficulty levels and realistic test sets. Ultimately, \vhop{} and \router{} provide an efficient and effective solution for visual agentic search that leaves native LLM capabilities entirely intact.
\end{abstract}

\begin{document}
\maketitle

\section{Introduction}

\begin{wrapfigure}{r}{0.5\textwidth}
\vspace{-38pt}
\centering
\includegraphics[width=\linewidth]{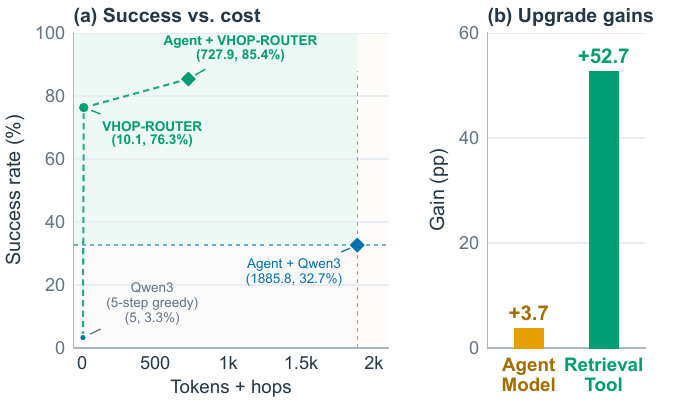}
\captionsetup{font=footnotesize,skip=3pt}
\caption{\textbf{Task performance and upgrade gains.}
Agent denotes Gemini~3.5~Flash.
(a) Success rate versus generated tokens plus retrieval hops.
Standalone Qwen3 uses five greedy retrieval steps.
(b) Gains over Agent + Qwen3 from upgrading the agent model
to Gemini~3.1~Pro or replacing the retrieval tool with
\router{}, \mbox{keeping the other component fixed.}}
\label{fig:vhop_results}
\vspace{-0.7cm}
\end{wrapfigure}

Agentic search equips large language model (LLM) agents with access to external knowledge bases, including up-to-date information and private data unavailable during training~\citep{yao2022react,trivedi2023interleaving}. Recent work fine-tunes LLMs to improve their reasoning and multi-step search~\citep{jin2025search}. These improvements still depend on a fixed retrieval tool to supply relevant evidence. Because the agent sees only the few results returned by this tool, it cannot reason over evidence it never receives.

This bottleneck is especially important in visual search, where each retrieved image may reveal the clue for the next step. When locating a misplaced item in a photo album, visual clues can be hard to express in a textual query (Figure~\ref{fig:retrieval_limitations}a) or be omitted from captions (Figure~\ref{fig:retrieval_limitations}b). A standard embedding retriever performs one retrieval step per call, leaving the agent to build the search chain from the returned images. In Figure~\ref{fig:retrieval_limitations}c, the retriever ranks the required next image 39th, outside the top-five results, leaving the red piggy bank needed for the following hop out of the agent's context. In contrast, embedding retrievers can score candidates across the entire indexed corpus.  

\vhopneedspace{25\baselineskip}
\WFclear % Do not carry the opening figure's wrapping onto the next page.
% Place beside the introduction's visual-search motivation.
% Requires graphicx, wrapfig, and caption (or subcaption).
\begin{wrapfigure}{r}
{0.5\textwidth}
\vspace{-0.5cm}
\centering
\includegraphics[width=\linewidth]{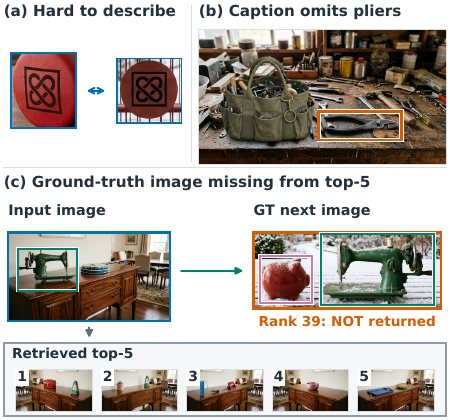}
\caption{\textbf{Visual retrieval bottlenecks.}
(a) Exact symbols can be hard to describe.
(b) The caption mentions ``metal tools'' but does not name the boxed pliers.
(c) Following the green sewing machine requires the ground-truth next image. The retriever ranks it 39th, so the red piggy bank needed for the following hop is not shown.}
\label{fig:retrieval_limitations}
\vspace{-0.5cm}
\end{wrapfigure}

We hypothesize that the main bottleneck in these scenarios is the retrieval tool itself. We therefore train this tool to navigate directly in visual embedding space over multiple steps, without agent fine-tuning. Shifting multi-hop search from the agent to the embedding model (replacing a single-step model with a multi-step one) significantly boosts task success rates while minimizing token length (Figure~\ref{fig:vhop_results}a). Upgrading to a stronger agent improves performance by only 3.7\%, whereas a multi-step embedding model yields a 52.7\% increase (Figure~\ref{fig:vhop_results}b). This validates that enhancing the retrieval tool is the far more effective approach. The following sections detail our data, evaluation, and training methods.

\begin{wraptable}{r}{0.5\textwidth}
    \centering
    \vspace{-0.6\baselineskip}
    \caption{\textbf{Comparison with related QA and visual-retrieval benchmarks.}
    I+T denotes image(s) and text. Visual hop clues reveal the next target in an intermediate image; planning involves avoiding or recovering from dead ends; symbolic verification checks answers against structured facts; training recipe indicates an evaluated solver-training procedure.}
    \label{tab:benchmark_comparison}
    \begingroup
    \fontsize{7.4}{8.5}\selectfont
    \setlength{\tabcolsep}{1.6pt}
    \renewcommand{\arraystretch}{1.02}

    \definecolor{benchmarkshade}{RGB}{235,242,249}
    \definecolor{benchmarkquiet}{RGB}{105,112,121}

    \newcommand{\benchcell}[1]{\begin{tabular}[c]{@{}c@{}}#1\end{tabular}}
    \newcommand{\benchyes}{$\checkmark$}
    \newcommand{\benchno}{\textcolor{benchmarkquiet}{--}}
    \newcommand{\benchdash}{\noalign{\vskip1.5pt
        \hbox to \linewidth{\color{benchmarkquiet}%
            \xleaders\hbox{\rule{2pt}{0.3pt}\hskip2pt}\hfill}%
        \vskip1.5pt}}

    \defcitealias{yang2018hotpotqa}{HotpotQA}
    \defcitealias{trivedi2022musique}{MuSiQue}
    \defcitealias{gong2025phantomwiki}{PhantomWiki}
    \defcitealias{chang2022webqa}{WebQA}
    \defcitealias{mensink2023encyclopedic}{Encyclopedic-VQA}
    \defcitealias{wu2025visual}{Visual Haystacks}
    \defcitealias{wang2025muirbench}{MuirBench}
    \defcitealias{johnson2017clevr}{CLEVR}
    \defcitealias{kordopatis2025ilias}{ILIAS}
    \defcitealias{grauman2022ego4d}{Ego4D-VQ}

    \resizebox{\linewidth}{!}{%
    \begin{tabular}{@{}lccccccc@{}}
        \toprule
        \textbf{Benchmark}
        & \textbf{Query}
        & \textbf{Output}
        & \benchcell{\textbf{Visual}\\\textbf{hop clues}}
        & \benchcell{\textbf{Planning}\\\textbf{ability}}
        & \benchcell{\textbf{Symbolic}\\\textbf{verification}}
        & \benchcell{\textbf{Difficulty}\\\textbf{controls}}
        & \benchcell{\textbf{Training}\\\textbf{recipe}} \\
        \midrule

        \citetalias{yang2018hotpotqa}
            & Text & Text & \benchno & \benchno & \benchno & \benchno & \benchyes \\
        \citetalias{trivedi2022musique}
            & Text & Text & \benchno & \benchno & \benchno & \benchno & \benchyes \\
        \citetalias{gong2025phantomwiki}
            & Text & Text & \benchno & \benchno & \benchyes & Depth, size & \benchno \\

        \benchdash

        \citetalias{chang2022webqa}
            & Text & Text & \benchno & \benchno & \benchno & \benchno & \benchyes \\
        \citetalias{mensink2023encyclopedic}
            & I+T & Text & \benchno & \benchno & \benchno & \benchno & \benchno \\
        \citetalias{wu2025visual}
            & Text & Yes/no & \benchno & \benchno & \benchno & Set size & \benchyes \\
        \citetalias{wang2025muirbench}
            & I+T & Option & \benchno & \benchno & \benchno & \benchno & \benchno \\
        \citetalias{johnson2017clevr}
            & I+T & Text & \benchno & \benchno & \benchyes
            & \benchcell{Scene/question\\complexity} & \benchyes \\

        \benchdash

        \citetalias{kordopatis2025ilias}
            & Image & Images & \benchno & \benchno & \benchno & \benchno & \benchyes \\
        \citetalias{grauman2022ego4d}
            & Image & Track & \benchno & \benchno & \benchno & \benchno & \benchyes \\

        \midrule
        \rowcolor{benchmarkshade}
        \textbf{\vhop{} (ours)}
            & I+T & Image & \benchyes & \benchyes & \benchyes
            & \benchcell{Hops, levels,\\corpus size}
            & \benchyes \\
        \bottomrule
    \end{tabular}
    }

    \endgroup
    \vspace{-0.5\baselineskip}
\end{wraptable}

To train and evaluate multi-step embedding models, we introduce \vhop{}, a flexible data generation framework for visual multi-hop image retrieval. It connects images via shared objects and offers controllable hop counts, object distractors, and varying task difficulty levels (Table~\ref{tab:benchmark_comparison}). For reliable training rewards and rigorous evaluation, verifiers use structured records of objects and relations to ensure the ground-truth labels satisfy all detailed query instructions and include all valid answers.

% Training and evaluation use disjoint object appearances and images to test generalization.

% extends text-based multi-hop QA~\citep{yang2018hotpotqa,trivedi2022musique} to visual search,

Building on these data and benchmarks, we develop \router{}, an end-to-end training framework that transforms a standard single-step embedding model into an autoregressive multi-step retriever. Training combines supervised fine-tuning (SFT), online imitation learning (Online~IL), and reinforcement learning with verifiable rewards (RLVR). At inference, \router{} autoregressively retrieves images and returns a linked image chain in a single tool call based on the initial query and completed hops, without agent text queries for intermediate hops (Figure~\ref{fig:vhop_framework}). Because LLM agents remain completely untouched, \router{} is compatible with any LLM-based agent and preserves all native LLM capabilities.

Through extensive experiments, we demonstrate that \router{} models achieve significant performance gains in both retrieval-only and agentic search tasks while maintaining superior efficiency. Furthermore, these models generalize well to unseen tasks across varying difficulty levels and hop counts, as well as to realistic test sets, ensuring the practical utility of the trained multi-step embedders.

Our contributions are four-fold:
\begin{itemize}[leftmargin=*,nosep,topsep=3pt]
    \item We formulate \emph{multi-hop image retrieval in visual latent space}, a novel task where an embedding model navigates independently in the latent embedding space to achieve multi-hop image retrieval.
    \item We introduce \vhop{}, a flexible and controllable data generation framework and benchmark. It provides verified ground-truth trajectories across varying, easily extendable difficulty levels.
    \item We develop \router{}, an end-to-end framework for training autoregressive multi-step embedding models.
    \item We demonstrate that offloading multi-step search from the agent to the retriever offers a highly effective and efficient solution for overcoming the bottleneck in agentic search, improving the success rate while significantly reducing token length and number of images in the LLM context.
\end{itemize}

\section{Related Work}
\label{sec:related}

Multi-hop Dense Retrieval~\citep{xiong2020answering} and Q-RAG~\citep{sorokin2026q} train retrievers for multi-step text search. \router{} shares this focus on learning the retrieval tool, but learns visual matching and recovery from dead ends. IRCoT~\citep{trivedi2023interleaving} and Search-R1~\citep{jin2025search} use LLM reasoning to guide retrieval; our policy retrieves image chains directly in latent space without generating intermediate text queries or fine-tuning LLMs. ILIAS~\citep{kordopatis2025ilias} and Ego4D's visual queries~\citep{grauman2022ego4d} search for a given object, whereas \vhop{} requires discovering the next target from intermediate images. WebQA~\citep{chang2022webqa} and Visual Haystacks~\citep{wu2025visual} evaluate question answering with visual evidence; our task focuses on retrieving the answer image through a chain of visual clues. Finally, CLEVR~\citep{johnson2017clevr} and PhantomWiki~\citep{gong2025phantomwiki} control reasoning within scenes and text corpora, respectively; \vhop{} controls retrieval paths across images, including hop counts and dead ends, and provides verified answers. Appendix~\ref{app:related} provides further discussion of these methods, benchmarks, and training approaches.

\begin{figure*}[t]
\centering
\includegraphics[width=\textwidth]{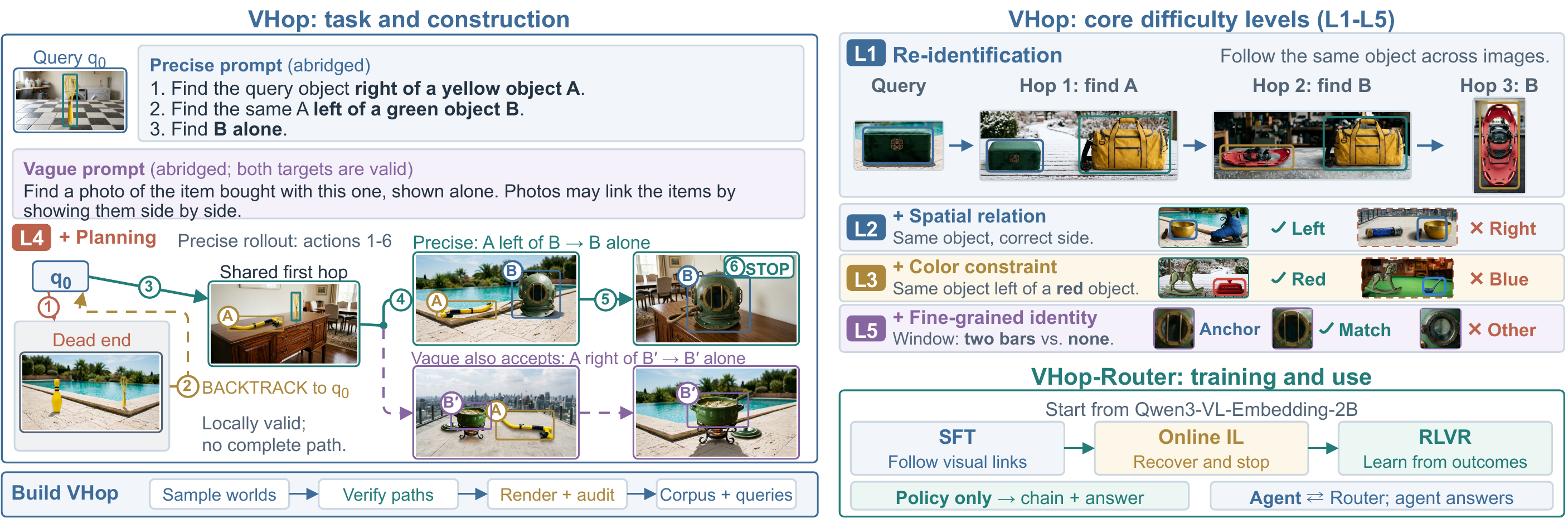}
\caption{\textbf{Overview of \vhop{} and \router{}.}
\textbf{Left:} An L4 search example with precise and vague prompts, dead ends, and backtracking, followed by benchmark construction.
\textbf{Right:} A complete L1 retrieval chain and small image comparisons for L2, L3, and L5, with the \router{} training stages below. Boxes, arrows, and detail crops are reader annotations. Examples of all difficulty levels appear in Figure~\ref{fig:vhop-levels-full} in Appendix~\ref{app:gen}.}
\vspace{-0.2cm}
\label{fig:vhop_framework}
\label{fig:vhop_levels}
\end{figure*}

\section{VHop: Controllable Multi-Hop Visual Search in Latent Space}
\label{sec:gen}

Consider that in real life, photo albums capture everyday moments and can help us locate misplaced objects. A multi-hop retriever can follow visual clues across these photos to find an image showing where an object was last seen, as illustrated in Figure~\ref{fig:natural-example}. We formalize this search process through \vhop{}, a controlled setting for studying multi-hop visual retrieval.

\textbf{Task definition.} Given a starting image and text instruction, the solver retrieves a sequence of images from a corpus, using each intermediate image to determine what to search for next. Each retrieval is a \emph{hop}, and the final retrieved image is the answer. In Figure~\ref{fig:vhop_framework}, the solver first retrieves an image in which the starting object appears to the right of a yellow object $A$, and then uses $A$ as the visual clue for the next hop. We consider two types of prompts. \emph{Precise prompts} specify each step using color and spatial constraints without directly naming the objects, providing a controlled setting for evaluating visual retrieval under explicit guidance. \emph{Vague prompts} omit these constraints and step-by-step instructions, reflecting settings in which users cannot specify the intermediate retrieval steps. We train and evaluate both settings; see details in Section~\ref{sec:exp}.

\paragraph{Controllable difficulty levels.}
\vhop{} provides five core difficulty levels (L1--L5; Figure~\ref{fig:vhop_levels}), each retaining earlier requirements while adding constraints or distractors. L1--L5 link images through the same physical object, with L4 introducing dead ends that require exploration and recovery. In Figure~\ref{fig:vhop_framework}, both a yellow tool and a yellow bowling pin appear to the left of the starting object, satisfying the first step, but only the tool permits a valid continuation. A solver choosing the bowling pin must therefore backtrack and try another candidate. The hop count, corpus size, and number of distractors are also configurable. Matching can also be extended to logos and printed text; Appendix~\ref{sec:ana-limits} describes these extensions and reports their results.

\paragraph{Symbolic generation, rendering, and verification.}
We first construct each image collection symbolically, specifying the object pool, object appearances and attributes, co-occurrence and spatial relations, the answer sequence, and distractors. For precise instructions, an independent verifier checks that exactly one sequence satisfies the query across the entire symbolic collection. If an instruction leaves relations unspecified (vague prompt), the verifier records all valid answers, any of which is accepted during evaluation. We then render the images and check the consistency of object identities, colors, and spatial relations. Training and evaluation use disjoint sets of object appearances. Appendices~\ref{app:gen} and~\ref{app:proofs} provide generation details and formal guarantees, including trajectory uniqueness for precise instructions; Section~\ref{sec:router_data} describes scoring rules and search budgets.

\section{\router{}: Training Framework}
\label{sec:router}

We developed \router{}, an end-to-end training framework that enables embedding models to navigate in visual latent spaces autoregressively through a three-stage process: Supervised fine-tuning (SFT) teaches valid retrieval chains, online imitation learning (Online~IL) teaches recovery and stopping, and reinforcement learning with verifiable rewards (RLVR) optimizes final-answer success. The trained policy can search independently and return image chains to VLM agents.

\vhopneedspace{10\baselineskip}
\subsection{Experimental Setup}
\label{sec:router_data}
\label{sec:exp}

\begin{wraptable}{r}{0.4\linewidth}
    \vspace{-0.5cm}
    \centering
    \caption{\textbf{Data allocation.} Query counts use precise three-hop prompts.}
    \label{tab:l4_data}
    \scriptsize
    \setlength{\tabcolsep}{3pt}
    \renewcommand{\arraystretch}{1.08}
    \begin{tabular}{@{}lcccc@{}}
        \toprule
        Split & Worlds & \shortstack{Queries/\\world} & \shortstack{Images/\\world} &
        \shortstack{Gold\\trajectories} \\
        \midrule
        \shortstack[l]{SFT/\\Online IL} & 1 & 978 & 4,492 & Yes \\
        RLVR       & 10 & 91--100 & 445--450 & No  \\
        \bottomrule
    \end{tabular}
\end{wraptable}

\paragraph{Training Data.}
All policy-training data are generated at L4, the first level with locally valid but globally incorrect branches, requiring exploration and backtracking. Data are organized into \emph{worlds}: a world is one independently generated image corpus together with the queries posed over it, and every query is answered using only the corpus of its own world (Appendix~\ref{app:gen}). SFT/Online~IL, RLVR, and evaluation use mutually disjoint worlds with distinct objects, backgrounds, and queries. Gold trajectories supervise SFT and Online~IL and are withheld for RLVR (Table~\ref{tab:l4_data}).

\paragraph{Evaluation setting.}
\begin{wraptable}{r}{0.2\linewidth}
    \centering
    \vspace{-0.5cm}
    \caption{\textbf{Evaluation sets.}}
    \label{tab:non_l4_eval_data}
    \renewcommand{\arraystretch}{1.08}
    \resizebox{\linewidth}{!}{%
        \begin{tabular}{@{}lrr@{}}
            \toprule
            Level & \shortstack{Queries/\\ world} &
            \shortstack{Images/\\ world} \\
            \midrule
            L1 & 99  & 300 \\
            L2 & 98  & 250 \\
            L3 & 97  & 350 \\
            L4 & 100 & 450 \\
            L5 & 100 & 550 \\
            \bottomrule
        \end{tabular}%
    }
    \vspace{-15pt}
\end{wraptable}
We train \router{} exclusively on L4 and evaluate it on L1--L5 without further fine-tuning. Each non-L4 level contains approximately 100 three-hop queries: L1--L3 test transfer to simpler settings, while L5 tests transfer with visually similar object instances. For L4, we use three independent evaluation worlds with 100 queries each for a more robust assessment. All evaluation worlds are disjoint from training, and non-L4 results measure whether the retrieval, backtracking, and stopping behaviors learned on L4 generalize across difficulty levels, corpora, and visual conditions. Table~\ref{tab:non_l4_eval_data} summarizes the statistics.

\paragraph{Baselines.}

We consider two types of baselines. \emph{(1) Retriever-only methods} search without an LLM agent. We evaluate the open-source Qwen3-VL-Embedding-2B (Qwen3) and the proprietary Gemini Embedding~2 (GE2)~\citep{shanbhogue2026gemini} under three strategies. \emph{Single-shot} retrieval encodes the original multimodal query and returns the top-5 images. \emph{Greedy} retrieval combines the original instruction with the latest retrieved image to select one new image at each of five steps. \emph{History-based} retrieval instead encodes the complete retrieval history. Caption-based variants use cached Gemini-3.1-Flash-Lite captions to test whether textual descriptions preserve the visual clues needed for multi-hop
retrieval. 

\emph{(2) Agentic-search methods} use an LLM to set up tool calls with observations~\citep{yao2022react}. We use Gemini~3.5~Flash with GE2 or Qwen3: the agent inspects the returned images and formulates its next query. To isolate the effect of the retrieval tool, we keep the agent fixed and replace its retriever with \router{}, which returns an image chain for the agent to inspect before answering or continuing the search. We also evaluate stronger agentic baselines by replacing Flash with Gemini~3.1~Pro (Section~\ref{sec:upgrade_agent}) or returning up to 50 images per retrieval call (Section~\ref{sec:ana-topk}).

\paragraph{Metrics.}
We evaluate both \emph{task performance} and \emph{search efficiency}. Task performance is measured by success rate. Embedding and caption-based baselines succeed if a valid answer appears among the five retrieved images. Standalone \router{} succeeds if it stops with a valid answer at the top of the stack, while agentic systems succeed if their declared final answer is valid. For efficiency, we report the total number of generated tokens, image inputs, and API calls per query. Agentic search allows up to  \(16\) tool calls and \(8{,}000\) output tokens per model call. Each \router{} rollout allows up to \(16\) actions; retrieval, backtracking, and stopping each count as one action.

\subsection{Autoregressive Retrieval Policy}
\label{sec:router_policy}

\paragraph{State and actions.}
Let $x$ be the text instruction, $q_0$ the query image, and $\mathcal C$ the image corpus. At decision $t$, the state consists of $x$ and the active stack $\sigma_t=(q_0,q_1,\ldots,q_{d_t})$, where $q_1,\ldots,q_{d_t}$ are the images retained after any backtracking and $d_t$ is the current hop count. The action space is $\{\textsc{Select}(a):a\in\mathcal C\}\cup\{\textsc{Backtrack},\textsc{Stop}\}$. \textsc{Select}$(a)$ appends image $a$ to the stack; \textsc{Backtrack} removes its last retrieved image; \textsc{Stop} ends the search and returns its current image stack. The query image $q_0$ always remains in the active stack.

\paragraph{\textsc{Select}}
A state encoder $E_s$ jointly encodes the instruction $x$ and active stack $\sigma_t$ as $s_t$, while an action encoder $E_a$ independently encodes each candidate image $a$:
\begin{equation}
    s_t=E_s(x,\sigma_t),\qquad e_a=E_a(a),\qquad
    \lVert s_t\rVert_2=\lVert e_a\rVert_2=1.
    \label{eq:router_encoders}
\end{equation}
Both encoders start from Qwen3-VL-Embedding-2B and use LoRA~\citep{hu2021lora}. We freeze $E_a$ after SFT and reuse its cached corpus embeddings during Online~IL, RLVR, and inference. Candidates are scored by cosine similarity, $z_t(a)=s_t^\top e_a$. After masking ineligible images, we select from the top-$K$ candidates $\mathcal C_t$ using a softmax with temperature $\tau$:
\begin{equation}
    \pi_\theta^{\mathrm{sel}}(a\mid s_t,\mathcal C_t)
    =\frac{\exp(z_t(a)/\tau)}
    {\sum_{a'\in\mathcal C_t}\exp(z_t(a')/\tau)}.
    \label{eq:select_policy}
\end{equation}

\paragraph{\textsc{Backtrack} and \textsc{Stop}.}
Two learned gates control stopping and backtracking. Each is a two-layer MLP over \([s_t;c_t]\), where \(c_t\in\mathbb{R}^4\) contains four retrieval statistics, such as the maximum candidate score (Appendix~\ref{app:online-il}). During stochastic rollouts, the gate probabilities are
\begin{equation}
    p_t^{\mathrm{stop}}
    =
    \operatorname{sigmoid}\!\left(
        g_{\mathrm{stop}}([s_t;c_t])/\tau_g
    \right),
    \qquad
    p_t^{\mathrm{bt}}
    =
    \operatorname{sigmoid}\!\left(
        g_{\mathrm{bt}}([s_t;c_t])/\tau_g
    \right),
    \label{eq:router_gates}
\end{equation}
Here, \(\tau_g\) is the gate temperature. Write $z_t^{\mathrm{stop}}=\mathbf{1}[a_t=\textsc{Stop}]$ and $z_t^{\mathrm{bt}}=\mathbf{1}[a_t=\textsc{Backtrack}]$ for the two action indicators. The full policy distribution is
\begin{equation}
    \pi_\theta(a_t\mid s_t)=
    \bigl(p_t^{\mathrm{stop}}\bigr)^{z_t^{\mathrm{stop}}}
    \Bigl[\bigl(1-p_t^{\mathrm{stop}}\bigr)
    \bigl(p_t^{\mathrm{bt}}\bigr)^{z_t^{\mathrm{bt}}}
    \Bigl(\bigl(1-p_t^{\mathrm{bt}}\bigr)\,
    \pi_\theta^{\mathrm{sel}}(a\mid s_t,\mathcal C_t)\Bigr)^{1-z_t^{\mathrm{bt}}}
    \Bigr]^{1-z_t^{\mathrm{stop}}}.
    \label{eq:router_policy}
\end{equation}
At evaluation time, the policy acts deterministically: it stops if \(p_t^{\mathrm{stop}}\geq 0.5\), otherwise backtracks if \(p_t^{\mathrm{bt}}\geq 0.5\), and otherwise selects the highest-scoring image.

A rollout $\rho=(a_0,\ldots,a_T)$ records all decisions, including selections
later undone by backtracking. Its prefix $\rho_{<t}$ determines the active
stack $\sigma_t$ before decision $t$. Its log-probability under $\pi_\theta$ is
\begin{equation}
    \log\pi_\theta(\rho\mid x,q_0)
    =\sum_{t=0}^{T}\log\pi_\theta(a_t\mid s_t).
    \label{eq:router_traj}
\end{equation}

\subsection{Training Pipeline}
\label{sec:router_training}

The three stages address successive requirements of sequential retrieval: ranking the correct continuation, recovering from mistakes, and optimizing final-answer success. 

\subsubsection{Supervised Fine-Tuning}
\label{sec:router_sft}

SFT teaches next-hop retrieval from gold trajectory prefixes. For a gold trajectory \(g_{1:H}\) and hop \(k\in\{1,\ldots,H\}\), the state encoder reads the instruction \(x\), query image \(q_0\), and prefix \(g_{1:k-1}\). The InfoNCE objective~\citep{oord2018representation} aligns the resulting state with the next gold image \(g_k\):
\begin{equation}
\begin{aligned}
    s_{k-1}
    &=
    E_s\!\left(x,(q_0,g_{1:k-1})\right),
    \qquad
    &
    \mathcal{L}_k
    &=
    -\log
    \frac{
        \exp\!\left(\langle s_{k-1},e_{g_k}\rangle/\tau\right)
    }{
        \sum_{a\in\mathcal{C}}
        \exp\!\left(\langle s_{k-1},e_a\rangle/\tau\right)
    }.
\end{aligned}
\label{eq:sft_loss}
\end{equation}
Training proceeds in two phases: we first freeze \(E_a\) and update only \(E_s\), then jointly optimize both encoders. This warm-up stabilizes training and prevents collapse (Figure~\ref{fig:sft}).

\subsubsection{Online Imitation Learning for Recovery and Stopping}
\label{sec:router_online_il}

Online~IL trains on states visited during search to teach recovery and stopping~\citep{ross2011reduction}. Each iteration collects fresh rollouts using the current policy, oracle guidance, and dead-end exploration. For each prefix $\rho_{<t}$, we recover the active stack $\sigma_t$, encode it as $s_t$, and obtain the oracle action $a_t^\star$. We freeze $E_a$ and update $E_s$ and the gates using the per-state loss in Eq.~\ref{eq:online_il_losses}. Only states collected in the current iteration are used for the update. Algorithm~\ref{alg:online_il} summarizes this procedure.

For oracle labels $y_t^{\mathrm{bt}}=\mathbf{1}[a_t^\star=\textsc{Backtrack}]$ and $y_t^{\mathrm{stop}}=\mathbf{1}[a_t^\star=\textsc{Stop}]$, the loss is
\begin{equation}
    \ell_t=-(1-y_t^{\mathrm{bt}})(1-y_t^{\mathrm{stop}})
    \log\pi_\theta^{\mathrm{sel}}(g_t\mid s_t,\mathcal C_t)
    +B_t(y_t^{\mathrm{bt}})+S_t(y_t^{\mathrm{stop}}).
    \label{eq:online_il_losses}
\end{equation}
The selection term applies only when the oracle selects image $g_t$, which is added to $\mathcal C_t$ if absent. The terms $B_t(y)$ and $S_t(y)$ are binary cross-entropy losses for the backtrack and stop gates (Eq.~\ref{eq:router_gates}), with target $y\in\{0,1\}$. Appendix~\ref{app:train} specifies their weights and action masks.

\begin{wrapfigure}{r}{0.38\textwidth}
\vspace{-1.5cm}
\begin{minipage}{\linewidth}
\captionsetup{hypcap=false}
\begin{algorithm}[H]
\caption{Online IL training}
\label{alg:online_il}
\small
\begin{algorithmic}[1]
\Require SFT checkpoint; oracle $\pi^\star$
\State Freeze action encoder $E_a$
\For{each training iteration}
    \State Collect fresh rollouts $\rho$
    \For{each rollout prefix $\rho_{<t}$}
        \State Replay $\rho_{<t}$ to recover $\sigma_t$
        \State Encode $s_t=E_s(x,\sigma_t)$
        \State Obtain $a_t^\star=\pi^\star(x,\sigma_t)$
    \EndFor
    \State Update $E_s$ and gates using $\ell_t$
\EndFor
\end{algorithmic}
\end{algorithm}
\end{minipage}
\vspace{-\intextsep}
\end{wrapfigure}

\subsubsection{Outcome-Based Reinforcement Learning}
\label{sec:router_rl}

Starting from the Online~IL checkpoint, RLVR trains $E_s$ and the two gate MLPs on the separate RLVR data in Table~\ref{tab:l4_data}, with $E_a$ fixed. For each query, we sample $G$ rollouts $\rho_g$ and assign $r_g=1$ only when the policy stops with a valid answer, and $0$ otherwise. Each rollout receives a leave-one-out advantage $A_g$~\citep{ahmadian2024back}. We use a PPO-style clipped surrogate~\citep{schulman2017proximal} with DAPO-style asymmetric clipping~\citep{yu2026dapo}. Algorithm~\ref{alg:rl} summarizes training, and Figure~\ref{fig:three_panels} shows its dynamics. Appendix~\ref{app:rl} details the loss terms, clipping, and KL penalties.

\vhopneedspace{8\baselineskip}
\section{Results}
\label{sec:router_results}
We evaluate combinations of four independent configurations. First, \textit{task difficulty} tests capability across query complexities. Second, we vary the \textit{number of hops}. Third, we compare \textit{precise prompts}---which give detailed instructions and constrain valid answers---against \textit{vague prompts} that omit step-by-step guidance to mimic real-world searches. Finally, we train with \textit{fixed} or \textit{mixed} hop counts. By default, training uses L4 (requiring planning and navigation around dead ends), 1--3 mixed hops, and precise prompts. We evaluate L1--L5 and ablate settings to test generalization, alongside a test set reflecting realistic user scenarios.

\subsection{Significant Improvements Across Difficulty Levels}
\label{sec:ana-vlm}

Table~\ref{tab:main} reports three-hop results across L1--L5 for standalone retrieval and agentic search. We train \router{} with precise or vague prompts, using either three-hop-only or mixed-hop data. On all these four training configurations, \router{} improved retrieval and agentic search across 5 difficulty levels significantly over baseline retrievers.

With our default mixed-hop training with precise prompts, standalone \router{} reaches \textbf{76.3\%} success on L4, compared with \textbf{3.7\%} for the strongest standard single-step retriever. It also improves agentic search: replacing Qwen3 with \router{} while keeping the agent fixed as Gemini~3.5~Flash raises success from \textbf{32.7\% to 85.4\% at in-domain L4 difficulty} and from \textbf{28.0\% to 84.0\% on L5 with zero-shot transfer}.

\begin{table*}[t]
\centering
\caption{\textbf{Three-hop success rate (\%) across L1--L5.}
Left: standalone retrieval. Right: Gemini-3.5-Flash with different tools. Blocks~(a) and~(b) use precise and vague prompts, respectively. Shaded columns show \router{} with three-hop-only and mixed-hop training; other baselines are shared across both regimes. Within each row, \textbf{bold} and \underline{underlining} mark the best and second-best scores on each side of the dashed rule; ties share a rank.}
\label{tab:main}
\scriptsize
\setlength{\tabcolsep}{2.6pt}
\renewcommand{\arraystretch}{1.0}
\resizebox{\textwidth}{!}{%
\begin{tabular}{@{}l *{12}{c} >{\columncolor{routerbg}}c >{\columncolor{routerbg}}c !{\dvrule} cc >{\columncolor{routerbg}}c >{\columncolor{routerbg}}c@{}}
\toprule
& \multicolumn{14}{c}{Retriever only } & \multicolumn{4}{!{\dvrule}c}{Agentic search} \\
\cmidrule(lr){2-15}\cmidrule(lr){16-19}
& \multicolumn{6}{c}{Embedding retrieval} & \multicolumn{6}{c}{Caption $\rightarrow$ embedding}
& \multicolumn{2}{>{\columncolor{routerbg}}c}{\router} & GE2 & Qwen3 & \multicolumn{2}{>{\columncolor{routerbg}}c@{}}{$+$\,\router} \\
\cmidrule(lr){2-7}\cmidrule(lr){8-13}\cmidrule(lr){14-15}\cmidrule(lr){18-19}
& \multicolumn{3}{c}{GE2} & \multicolumn{3}{c}{Qwen3} & \multicolumn{3}{c}{GE2} & \multicolumn{3}{c}{Qwen3}
& $3$-hop & mixed & & & $3$-hop & mixed \\
\cmidrule(lr){2-4}\cmidrule(lr){5-7}\cmidrule(lr){8-10}\cmidrule(lr){11-13}
Level & 1-shot & Greedy & Hist. & 1-shot & Greedy & Hist.
& 1-shot & Greedy & Hist. & 1-shot & Greedy & Hist. & & & & & & \\
\midrule
\multicolumn{19}{@{}l}{\textbf{(a) Precise prompt}}\\
\midrule
L1 (re-ID) & 2.0 & 6.1 & 2.0 & 5.1 & 9.1 & 6.1 & 1.0 & 0.0 & 1.0 & 6.1 & 4.0 & 3.0 & \textbf{66.7} & \underline{61.6} & 60.6 & 73.7 & \textbf{85.9} & \underline{80.8} \\
L2 (spatial) & 4.1 & 5.1 & 4.1 & 4.1 & 8.2 & 4.1 & 1.0 & 2.0 & 1.0 & 6.1 & 2.0 & 3.1 & \underline{43.9} & \textbf{59.2} & 58.2 & \underline{69.4} & 62.2 & \textbf{80.6} \\
L3 (color) & 3.1 & 5.2 & 2.1 & 3.1 & 7.2 & 2.1 & 1.0 & 0.0 & 0.0 & 0.0 & 1.0 & 0.0 & \underline{44.3} & \textbf{85.6} & 56.7 & \underline{77.3} & 58.8 & \textbf{89.7} \\
L4 (dead ends) & 1.3 & 3.7 & 1.7 & 1.0 & 3.3 & 1.0 & 0.0 & 0.0 & 0.0 & 0.7 & 1.0 & 0.3 & \underline{48.0} & \textbf{76.3} & 19.3 & 32.7 & \underline{59.7} & \textbf{85.4} \\
L5 (twins) & 1.0 & 2.0 & 2.0 & 3.0 & 4.0 & 2.0 & 0.0 & 0.0 & 0.0 & 1.0 & 1.0 & 1.0 & \underline{42.0} & \textbf{72.0} & 15.0 & 28.0 & \underline{57.0} & \textbf{84.0} \\
\midrule
\multicolumn{19}{@{}l}{\textbf{(b) Vague prompt}}\\
\midrule
L1 (re-ID) & 1.0 & 1.0 & 1.0 & 3.0 & 5.1 & 2.0 & 0.0 & 0.0 & 1.0 & 3.0 & 2.0 & 2.0 & \textbf{46.5} & \underline{41.4} & 11.1 & 9.1 & \textbf{67.7} & \underline{45.5} \\
L2 (spatial) & 2.0 & 4.1 & 6.1 & 0.0 & 8.2 & 2.0 & 2.0 & 2.0 & 2.0 & 0.0 & 2.0 & 2.0 & \textbf{55.1} & \underline{36.7} & 6.1 & 0.0 & \textbf{77.6} & \underline{51.0} \\
L3 (color) & 0.0 & 6.2 & 4.2 & 4.2 & 6.2 & 4.2 & 2.1 & 4.2 & 0.0 & 0.0 & 0.0 & 0.0 & \textbf{56.2} & \underline{29.2} & 6.2 & 6.2 & \textbf{72.9} & \underline{33.3} \\
L4 (dead ends) & 2.0 & 1.3 & 1.3 & 0.7 & 0.0 & 0.7 & 0.0 & 2.0 & 0.0 & 0.0 & 0.0 & 0.0 & \underline{38.0} & \textbf{42.0} & 6.0 & 2.0 & \textbf{55.3} & \underline{50.7} \\
L5 (twins) & 4.0 & 2.0 & 2.0 & 2.0 & 0.0 & 0.0 & 0.0 & 0.0 & 0.0 & 0.0 & 0.0 & 0.0 & \textbf{42.0} & \underline{24.0} & 0.0 & 2.0 & \textbf{50.0} & \underline{32.0} \\
\bottomrule
\end{tabular}%
}
\end{table*}

\subsection{Retriever Upgrades Yield 14x the Gain of Agent Upgrades}
\label{sec:upgrade_agent}

\begin{wraptable}{l}{0.3\textwidth}
\vspace{-0.5cm}
  \centering
  \caption{\textbf{Upgrading agent vs.\ upgrading tool}}
  \label{tab:upgrade_agent}
  \scriptsize
  \setlength{\tabcolsep}{2.6pt}
  \renewcommand{\arraystretch}{1.0}
  \resizebox{\linewidth}{!}{%
    \begin{tabular}{@{}lcc>{\columncolor{routerbg}}c@{}}
      \toprule
      & \multicolumn{2}{c}{Qwen3}
      & \router{} \\
      \cmidrule(lr){2-3}
      \cmidrule(lr){4-4}
      Level & Flash & Pro & Flash \\
      \midrule
      L1 & 73.7 & \underline{78.8}
         & \textbf{80.8} \\
      L2 & 69.4 & \underline{77.6}
         & \textbf{80.6} \\
      L3 & 77.3 & \underline{\textbf{90.7}}
         & 89.7 \\
      L4 & 32.7 & \underline{36.3}
         & \textbf{85.4} \\
      L5 & 28.0 & \underline{34.0}
         & \textbf{84.0} \\
      \bottomrule
    \end{tabular}%
  }
  % \vspace{-1cm}
\end{wraptable}
On L4 and L5, upgrading the tool yields larger gains than upgrading the
agent. Replacing Flash with Gemini~3.1~Pro while keeping Qwen3 adds only
3.7 percentage points on L4, compared with 52.7 points upgrading tool
from Qwen3 to \router{} (Figure~\ref{fig:vhop_results}b,
Table~\ref{tab:upgrade_agent}). On L5, the corresponding gains are
6.0 and 56.0 points. Appendix Figure~\ref{fig:vlm-full} reports both
tool training configurations across all levels.

\subsection{Can Retrieving More Images Match a Learned Multi-Step Retriever?}
\label{sec:ana-topk}

Agentic baselines in Table~\ref{tab:main} return only the top-ranked image (\(k=1\)) per call. Could a single-step retriever close the gap by returning more candidates for the agent to reason over? We evaluate GE2 and Qwen3 at \(k\in\{1,3,5,10,50\}\), keeping Gemini~3.5~Flash fixed. The L4 test corpus contains \(450\) images, so \(k=50\) exposes more than one ninth per call. Increasing \(k\) improves success, but neither retriever matches \router{} (\(85.4\%\)) even at \(k=50\) (Figure~\ref{fig:topk-analysis}). Visual context also grows rapidly. With Qwen3 at $k=50$, the agent averages $275$ retrieved images in context per question and a cumulative API payload of $1{,}163$ image inputs across turns. In contrast, Flash with \router{} averages $2.1$ tool calls, $12$ returned images, and a cumulative API payload of $33$ images. It thus \textbf{outperforms while requiring $~23\times$ fewer images in context and a $35\times$ smaller cumulative API payload} than the $k=50$ baselines. All results use precise three-hop test queries; Appendix Table~\ref{tab:topk-full} includes the three-hop-trained tool and other levels.

\begin{figure*}[ht]
    \centering
    \captionsetup{skip=3pt}
    \begin{subfigure}[t]{0.32\textwidth}
        \vspace{0pt}
        \centering
        \includegraphics[width=\linewidth]{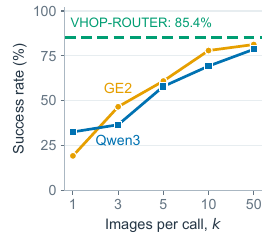}
        \caption{Success rate.}
        \label{fig:topk}
    \end{subfigure}
    \hfill
    \begin{subfigure}[t]{0.32\textwidth}
        \vspace{0pt}
        \centering
        \includegraphics[width=\linewidth]{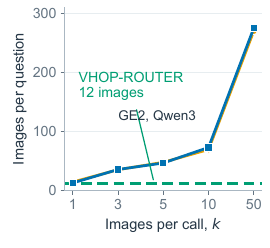}
        \caption{Images in context.}
        \label{fig:topk-context}
    \end{subfigure}
    \hfill
    % Preserve the outer figure counter across the embedded table caption.
    \edef\vhopfigurecaptionflags{\arabic{caption@flags}}%
    \begin{minipage}[t]{0.32\textwidth}
        \vspace{0pt}
        \centering
        \fontsize{6.5}{7.5}\selectfont
        \setlength{\tabcolsep}{2pt}
        \renewcommand{\arraystretch}{1.05}
        \resizebox{\linewidth}{!}{%
        \begin{tabular}{@{}lrrrr@{}}
            \toprule
            Retrieval tool & \(k\) & Calls & Imgs & Sent \\
            \midrule
            \multirow{5}{*}{GE2}
              & 1  & 14.5 & 15  & 120  \\
              & 3  & 11.8 & 35  & 249  \\
              & 5  & 9.6  & 48  & 300  \\
              & 10 & 6.9  & 69  & 344  \\
              & 50 & 5.4  & 270 & 1,084 \\
            \midrule
            \multirow{5}{*}{Qwen3}
              & 1  & 13.0 & 13  & 104  \\
              & 3  & 11.9 & 36  & 255  \\
              & 5  & 9.3  & 47  & 295  \\
              & 10 & 7.3  & 73  & 392  \\
              & 50 & 5.5  & 275 & 1,163 \\
            \midrule
            \router{}
              & chain & 2.1 & 12 & 33 \\
            \bottomrule
        \end{tabular}%
        }
        \captionof{table}{Tool call usage.}
        \label{tab:topk}
    \end{minipage}%
    \setcounter{caption@flags}{\vhopfigurecaptionflags}

\caption{\textbf{Wider retrieval does not outperform \router{}.}
Gemini~3.5~Flash uses precise three-hop L4 test queries; \router{} uses precise mixed-hop training.
\textbf{Left:} Success rate; dashed line: Flash with \router{}.
\textbf{Middle:} Returned images per question.
\textbf{Right:} Mean calls, returned images (\emph{Imgs}), and image occurrences across agent requests (\emph{Sent}), per question.}
    \label{fig:topk-analysis}
    \vspace{-0.5cm}
\end{figure*}

\subsection{Generalization on One- and Two-Hop Queries}
\label{sec:ana-hops}

We evaluate \router{} on one- and two-hop queries across L1--L5, using the same checkpoints and held-out worlds as in the three-hop evaluation. Under the default precise mixed-hop configuration, the standalone \router{} achieves a \textbf{99.7\%} success rate on one-hop L4 queries, showing strong performance on single-step retrieval. Furthermore, it reaches \textbf{90.0\%} success on two-hop L4 queries. Together, these results show that one mixed-hop policy handles both shorter query types without further fine-tuning. Comprehensive results for all four training configurations and evaluation levels are detailed in Appendix Tables~\ref{tab:main-h1} and~\ref{tab:main-h2}.
\par

% Preamble: \usepackage{wrapfig}

\subsection{Robustness to Additional Distractor Objects}
\label{sec:ana-three-objects}

In the standard evaluation (Table~\ref{tab:main}), each hop image contains two foreground objects. To test robustness, we add a third object to every hop image in the L4 test set. These objects introduce competing visual cues without changing the answers. This modification substantially increases retrieval difficulty (Figure~\ref{fig:three_objects}): success for the strongest standalone policy, \router{} with precise mixed-hop training, drops from \(76.3\%\) to \(52.2\%\). However, it still exceeds the strongest agentic baseline with a single-step retriever (\(27.4\%\)) by \(24.8\) percentage points, retaining a substantial advantage under this unseen visual shift.

\begin{figure*}[ht]
    \centering
    \captionsetup{skip=3pt}
    \begin{subfigure}[t]{0.27\textwidth}
        \vspace{0pt}
        \centering
        \includegraphics[width=\linewidth]{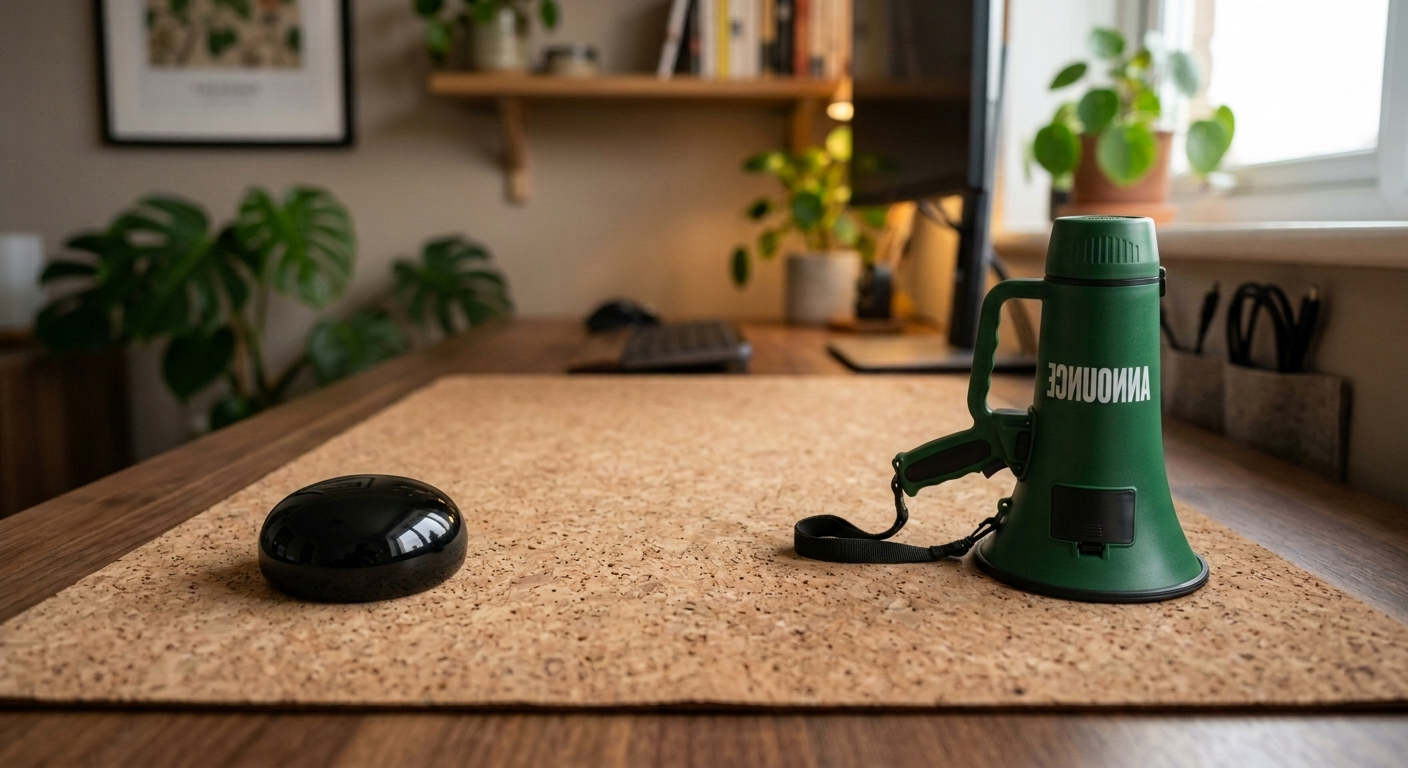}
        \caption{Standard two-object hop image.}
        \label{fig:probe-hop1-l4}
    \end{subfigure}
    \hfill
    \begin{subfigure}[t]{0.27\textwidth}
        \vspace{0pt}
        \centering
        \includegraphics[width=\linewidth]
            {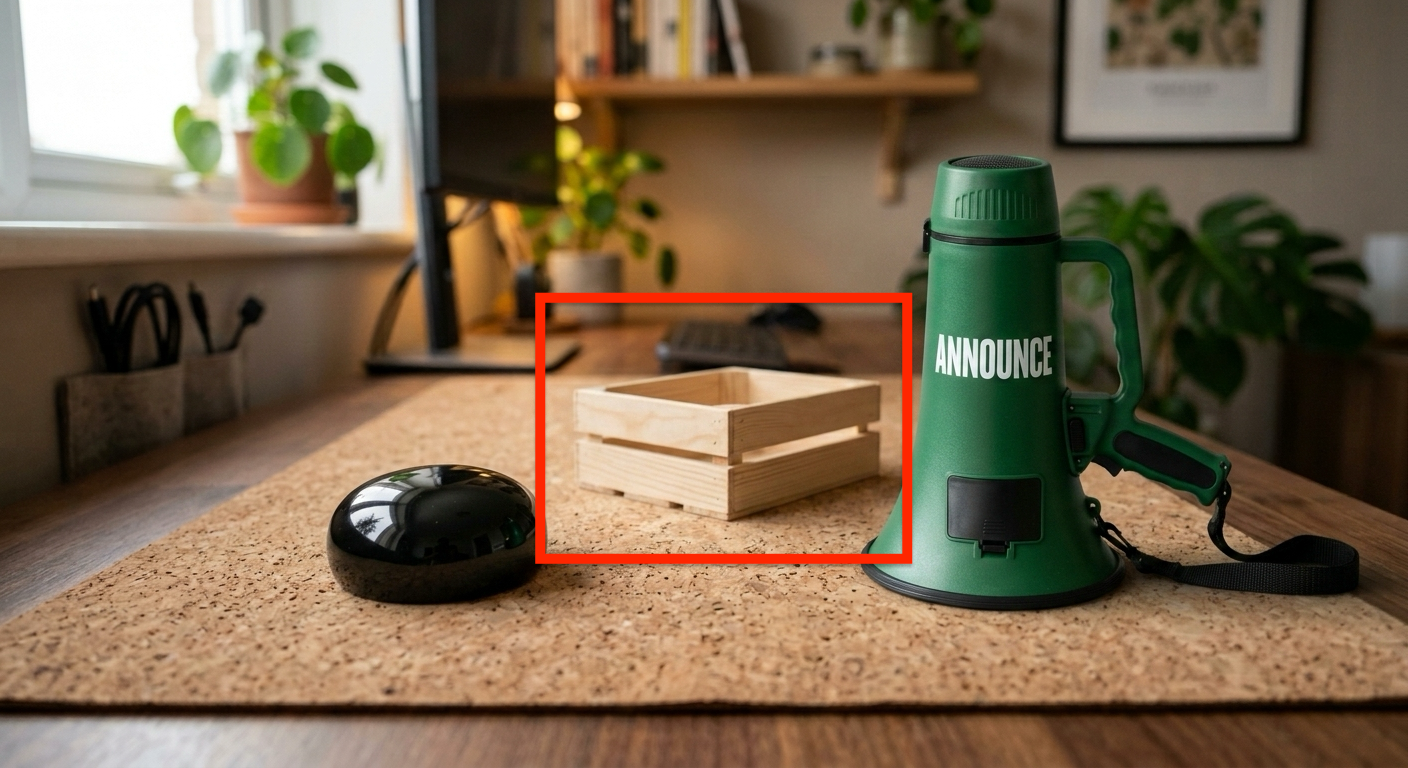}
        \caption{The same setting with an added distractor.}
        \label{fig:probe-hop1-l4d3}
    \end{subfigure}
    \hfill
    % Preserve the outer figure counter across the embedded table caption.
    \edef\vhopfigurecaptionflags{\arabic{caption@flags}}%
    \begin{minipage}[t]{0.42\textwidth}
        \vspace{-5pt}
        \centering
        \captionsetup{type=table}
        \caption{\textbf{L4 answer accuracy (\%).}
        \router{} policy alone.}
        \label{tab:robust}

        \renewcommand{\arraystretch}{1.05}
        \resizebox{\linewidth}{!}{%
    \begin{tabular}{@{}llrr@{}}
        \toprule
        Method & Training configuration
            & \(2\) objects & \(3\) objects \\
        \midrule
        \multirow{4}{*}{\router}
            & Precise, \(3\)-hop
            & 48.0 & 30.4 \\
            & Precise, mixed-hop
            & \textbf{76.3} & \textbf{52.2} \\
            & Vague, \(3\)-hop
            & 38.0 & 22.0 \\
            & Vague, mixed-hop
            & 42.0 & 17.3 \\
        \midrule
        Flash \(+\) GE2
            & \multicolumn{1}{c}{-}
            & 19.3 & 16.4 \\
        Flash \(+\) Qwen3
            & \multicolumn{1}{c}{-}
            & 32.7 & 27.4 \\
        \bottomrule
    \end{tabular}%
}
    \end{minipage}%
    \setcounter{caption@flags}{\vhopfigurecaptionflags}

    \caption{\textbf{Robustness to a foreground distractor.}
    \textbf{Left:} A standard \vhop{} image.
    \textbf{Middle:} A variant with a third, unseen distractor.
    \textbf{Right:} Retrieval success with distractor.}
    \label{fig:three_objects}
    \vspace{-0.5cm}
\end{figure*}

\subsection{Generalization to Realistic Queries and Settings}
\label{sec:ana-natural}

\setlength{\intextsep}{4pt}
\begin{wraptable}{r}{0.48\textwidth}
    \centering
    \captionsetup{skip=3pt}
    \caption{\textbf{Task success on realistic queries.}
    Answer found among returned images (\%).}
    \label{tab:natural}
    \renewcommand{\arraystretch}{1.05}
    \resizebox{\linewidth}{!}{%
        \begin{tabular}{@{}llr@{}}
            \toprule
            Agentic Search Method & Training configuration  & Success rate \\
            \midrule
            \multirow{2}{*}{Flash \(+\) \router{}}
                & Precise, \(3\)-hop & \textbf{36.0} \\
                & Precise, mixed-hop & 34.0 \\
            \midrule
            Flash \(+\) GE2
                &  \multicolumn{1}{c}{-} & 22.0 \\
            Flash \(+\) Qwen3
                &  \multicolumn{1}{c}{-} & 26.0 \\
            \bottomrule
        \end{tabular}%
    }
    \vspace{-0.2cm}
\end{wraptable}

All preceding experiments use procedurally generated images and prompts. To test the generalization capability in real-world scenarios, we curate another test set with realistic queries and settings including \(50\) questions and \(386\) images. Annotators write the questions and specify everyday scenes. The images are re-rendered to remove privacy-sensitive details while preserving the \mbox{intended} content. Each query requires following visual clues to locate a target item, with an image of the corresponding box serving as the final answer (Figure~\ref{fig:natural-example}).

Without additional training, Flash with \router{} retrieves the correct answer among its returned images on \textbf{36.0\%} of queries with precise three-hop training and \textbf{34.0\%} with precise mixed-hop training, compared with \textbf{22.0\%} for Flash with GE2 and \textbf{26.0\%} with Qwen3 (Table~\ref{tab:natural}). Full results, including standalone \router{}, appear in Appendix~\ref{app:human-authored}.

\begin{figure}
    \centering
    \includegraphics[width=\linewidth]{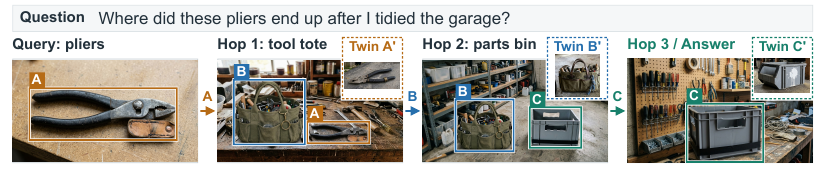}
    \caption{\textbf{A realistic retrieval task.}
    Dashed insets show corpus distractors depicting the same object types but different physical instances. Zoom in for details.}
    \label{fig:natural-example}
    \vspace{-0.65cm}
\end{figure}

\section{Analysis}
\subsection{Controlling Rollout Length in RLVR}
\label{sec:ana-lenpen}

\begingroup
\setlength{\intextsep}{4pt}
\begin{wrapfigure}{r}{0.49\textwidth}
    \centering
    \captionsetup{skip=3pt}
    \includegraphics[width=\linewidth]{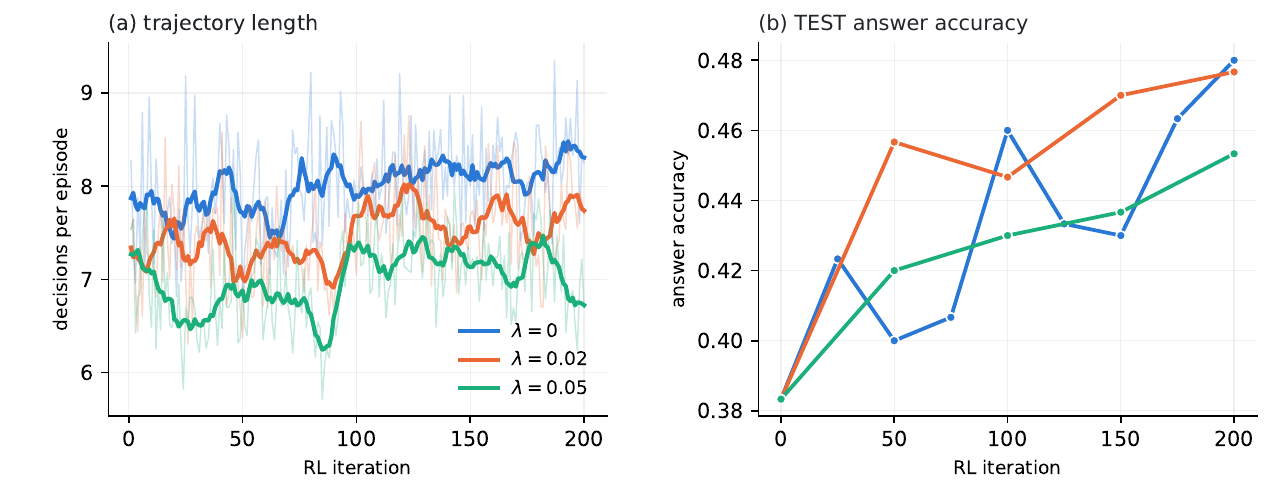}
    \caption{\textbf{RLVR length-penalty dynamics.}
    \textbf{(a)} Mean $|\rho|$.
    \textbf{(b)} L4 test success rate.}
    \label{fig:lenpen}
\end{wrapfigure}

The default RLVR objective rewards only answer correctness, without a search cost (Section~\ref{sec:router_rl}). Training thus lengthens rollouts and increases budget exhaustion (Figure~\ref{fig:three_panels}). To test whether a length penalty shortens rollouts, we anneal the terminal reward as
$
    r_{\lambda}
    =
    r-\lambda |\rho|
$, $\lambda\in\{0,0.02,0.05\}$.

Increasing \(\lambda\) consistently shortens rollouts (Figure~\ref{fig:lenpen}). On the L4 test set, \(\lambda=0.02\) achieves accuracy comparable to the unpenalized objective, whereas \(\lambda=0.05\) degrades performance. A modest penalty thus removes unnecessary actions without sacrificing accuracy; a stronger penalty affects the accuracy reward. Further results appear in Appendix~\ref{sec:appendix:lenpen} and Table~\ref{tab:lenpen}.
\WFclear
\endgroup

\subsection{Trajectory Pattern Analysis}
\label{sec:trajectory-analysis}

Representative trajectories in Appendix~\ref{app:cases} illustrate the contributions of each training stage and the agent. \textbf{SFT} improves visual matching, retrieving the correct first hop where the untrained model fails, but can still follow an incorrect spatial relation. \textbf{Online~IL} learns backtracking and stopping but may stop on a wrong branch. \textbf{RLVR} refines these decisions through final-answer rewards, recovering from the same initial mistake as Online~IL, completing the correct chain, and stopping at the answer (Figure~\ref{fig:case_a_rl}). The agent further complements \router{} through \textbf{answer selection} and \textbf{re-querying}: it selects the correct answer from the returned chain when \router{} continues searching until its budget runs out (Figure~\ref{fig:answer-selection}), or revises the query and calls the same policy again when the answer is missing, recovering the complete correct image chain (Figure~\ref{fig:case_b_agent_requery}).

\section{Conclusion and Limitations}
\label{sec:conclusion}

We address visual agentic search bottlenecks by offloading multi-hop search directly to the retrieval tool. We introduce \vhop{}, a benchmark and data generation framework, alongside \router{}, an end-to-end training pipeline for autoregressive multi-step embedders. \router{} significantly boosts retrieval and agentic search success rates while drastically reducing LLM context sizes and API payloads and preserving all native LLM capabilities. Scaling the data generation framework to more objects remains challenging because it relies on explicitly defined relationships. The trained policy also struggles to generalize beyond the hop counts seen during training. Future work will focus on scaling both \vhop{} and \router{} and incorporating real-world datasets.

\section{Acknowledgment}

We extend our sincere gratitude to Shanfeng Zhang, Mojtaba Seyedhosseini, Howard Zhou, and Tom Duerig for providing essential resources, alongside their invaluable guidance and support throughout this project. Additionally, we thank Shengcao Cao for reviewing the manuscript and offering insightful technical feedback.

\FloatBarrier
\subsubsection*{Reproducibility statement}
The stdlib-only generator and verifier are deterministic from stored seeds and $\theta$.
Appendix~\ref{app:gen} specifies rendering checks, data splits, and prompts;
Appendix~\ref{app:train} gives training procedures and hyperparameters.
Section~\ref{sec:exp} defines baselines and metrics.

\subsubsection*{AI use statement}
We used AI tools to assist with manuscript drafting and editing, figure preparation, and reference checking.
Generative models were also used to render benchmark images (Appendix~\ref{app:gen}) and produce captions for retrieval baselines (Section~\ref{sec:router_data}).
The authors take responsibility for the final content, including the text, figures, references, and experimental claims.

\bibliography{iclr2027_conference}

\appendix
% The appendix is split by topic; main-text inputs are intentionally unchanged.
\clearpage
\section{Additional Related Work}
\label{app:related}
\paragraph{Multi-hop retrieval and question answering.}
HotpotQA~\citep{yang2018hotpotqa} requires combining evidence from multiple
Wikipedia articles, while MuSiQue~\citep{trivedi2022musique} composes
single-hop questions into connected reasoning chains and filters out
shortcuts. Multi-hop Dense Retrieval~\citep{xiong2020answering} learns to retrieve
a passage conditioned on earlier evidence.
IRCoT~\citep{trivedi2023interleaving} interleaves retrieval with reasoning so that
each informs the next step. VHOP studies a visual form of this dependency: a retrieved image supplies
the object or mark needed to find the next image, and the final answer is
itself an image.

\paragraph{Visual retrieval and memory.}
CLIP~\citep{radford2021learning} learns a shared image--text embedding space
that supports retrieval by similarity. Instance-level benchmarks such as
ILIAS~\citep{kordopatis2025ilias} test whether a system can find the
same particular object among many distractors. Ego4D's visual-query
task~\citep{grauman2022ego4d} uses an object image to locate its last
appearance in an egocentric video. VHOP extends such visual search to a chain of linked objects: the starting image may not show the final
target, so directly matching that image to the answer is insufficient.
The solver must use intermediate images to determine what to retrieve next.

\paragraph{Multimodal and multi-image reasoning.}
WebQA~\citep{chang2022webqa} combines evidence retrieval from image or text
sources with natural-language answer generation.
Encyclopedic-VQA~\citep{mensink2023encyclopedic} asks detailed questions
about visually identified entities and provides a Wikipedia knowledge base.
Visual Haystacks~\citep{wu2025visual} tests question answering over large
image collections, including questions that require evidence from multiple
images. MuirBench~\citep{wang2025muirbench} covers a broad range of
multi-image tasks, including visual retrieval. VHOP focuses on a specific question: can a system repeatedly recover a new visual clue,
follow it through a corpus, and recover when a locally plausible match
does not lead to an answer? Its generator makes these dependencies and
alternative branches explicit for controlled evaluation.

\paragraph{Controlled evaluation.}
CLEVR~\citep{johnson2017clevr} uses generated scenes and executable
question programs to diagnose compositional visual reasoning.
GQA~\citep{hudson2019gqa} generates compositional questions from scene
graphs of real images. PhantomWiki~\citep{gong2025phantomwiki} generates
consistent document collections and questions with adjustable reasoning
difficulty and corpus size. VHOP follows this approach of making task
structure controllable, while placing the evidence across separate images.
Its verifier checks satisfying trajectories over the pooled corpus before
rendering. The guarantees are conditional on faithful rendering; the
additional retrieval bounds also require the solver restrictions and
sampling assumptions stated in Appendix~\ref{app:proofs}.
Separately, \citet{weller2026theoretical} establish limitations imposed by
embedding dimension on the relevance sets a single-vector retriever can
represent. Our analysis concerns missing intermediate evidence and
locally ambiguous branches, rather than embedding dimension.

\paragraph{Agents and learned search policies.}
ReAct~\citep{yao2022react} interleaves language reasoning with actions and
external observations. Search-R1~\citep{jin2025search} trains language
models to issue search queries using outcome-based reinforcement learning.
Q-RAG~\citep{sorokin2026q} trains embeddings for multi-step text
retrieval through a value-based objective. Our method also adapts retrieval
representations. It operates on image sequences and learns selection,
backtracking, and stopping from imitation and final-answer rewards.
The trained policy can also serve as a tool for an agent that revises
queries or selects an answer from returned images.

\paragraph{Imitation learning and outcome supervision.}
Collecting expert labels on states visited by the learner addresses the
distribution shift that motivates DAgger~\citep{ross2011reduction}.
Our Online~IL stage uses this idea but updates only on the current
iteration's states, without accumulating a dataset across iterations;
we therefore do not call the algorithm DAgger or invoke its guarantees.
The subsequent RLVR stage uses a terminal answer check, a leave-one-out
baseline~\citep{ahmadian2024back}, and a clipped policy
objective~\citep{schulman2017proximal,yu2026dapo}. We adapt these components to a visual retrieval policy with explicit
selection, backtracking, and stopping decisions.

\section{Benchmark Generation and Verification}
\label{app:gen}

\vhop{} separates task construction from image rendering. The generator
first specifies objects, their links across images, and the constraints in
each question. A symbolic verifier checks the resulting search problem over
the entire corpus. Rendering then turns this specification into images.
This separation provides exact task labels while allowing the visual
appearance and search difficulty to vary.

\subsection{Difficulty Levels and World Construction}

\begin{table}[ht]
\centering
\caption{\textbf{Cumulative difficulty levels.} Each level retains the
requirements of the preceding levels.}
\label{tab:ladder}
\small
\setlength{\tabcolsep}{4pt}
\renewcommand{\arraystretch}{1.1}
\begin{tabularx}{\linewidth}{@{}llX@{}}
\toprule
Level & Added requirement & Required behavior \\
\midrule
L1 & Object identity & Match the same physical object across views. \\
L2 & Spatial relations & Reject an otherwise matching scene with the wrong relation. \\
L3 & Color constraints & Check the color specified in the instruction. \\
L4 & Dead ends & Recover after a plausible branch fails to continue. \\
L5 & Appearance twins & Distinguish instances with similar type, color, and appearance. \\
L6 & Shared logos & Follow links defined by a shared logo. \\
L7 & Shared text & Follow links defined by a shared printed string. \\
\bottomrule
\end{tabularx}
\end{table}

\begin{figure}[ht]
\centering
\includegraphics[width=\textwidth]{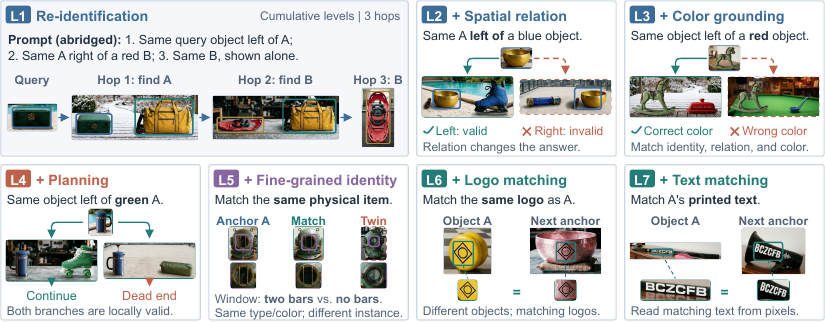}
\caption{\textbf{Visual examples of the difficulty levels.}
L1 shows a complete retrieval chain. L2 and L3 add spatial and color constraints;
L4 introduces dead ends; and L5 requires distinguishing look-alike objects.
L6 and L7 extend the matching rule to logos and printed text.
The examples come from separate questions. Boxes, arrows, and enlarged details
are reader annotations.}
\label{fig:vhop-levels-full}
\end{figure}

A \emph{world} is one image corpus and its associated questions. A
\emph{chain family} contains a gold chain and its related distractors.
The generator controls hop count, matching rules, relations, colors,
family count, and the numbers of distractors and forks. It assigns fresh
matching keys to each family and records every image containing each key.
The final portrait is included in the searchable corpus. Query images are
fresh views of the starting objects.

Distractors target different errors. Relation and color decoys violate a
stated constraint. A dead end satisfies the local match but lacks a valid
continuation. An appearance twin resembles the correct object while
representing a different instance. Extra portraits share query-specified
properties with the answer, preventing those properties alone from
identifying it. Images and filenames are shuffled without exposing their
roles. The corpus is available from the start; the task requires discovering
the links between images, rather than revealing hidden corpus entries.

\begin{algorithm}[ht]
\caption{Generating and verifying a \vhop{} world}
\label{alg:gen}
\small
\begin{algorithmic}[1]
\Require Difficulty level, hop count, family count, and object/mark pools
\For{each chain family}
    \State sample matching rules, spatial relations, and color constraints
    \State assign fresh keys and construct the gold scenes and final portrait
    \State add the level's decoys, dead ends, twins, and matching distractor portraits
    \State construct a precise query without naming the intermediate key values
\EndFor
\State pool the images and shuffle their identifiers
\For{each query}
    \State enumerate all satisfying trajectories over the pooled corpus
    \State check the unique precise path and the intended distractor structure
    \State enumerate acceptable endpoints for the vague variant, if supported
    \State reject and regenerate any family that fails verification
\EndFor
\State render the verified world and check its images against the specification
\State \Return image corpus, query images, instructions, and verified labels
\end{algorithmic}
\end{algorithm}

\subsection{Prompt Variants and Data Separation}

Precise instructions state the sequence of matching rules, relations, and
filters. They refer to an intermediate object through the preceding image,
without giving its identity, logo value, or printed string in advance.
Vague instructions omit a relation and may admit more than one answer.
The verifier enumerates the acceptable endpoints after that omission;
evaluation accepts any member of this set. Vague variants are reported for
L1--L5. One- and two-hop questions are verified suffixes of the three-hop
chains and use the same image corpora.

Training and evaluation use separate worlds and separate pools of object
appearances and marks. SFT and Online~IL use worlds with full trajectory
labels. RLVR uses a separate practice pool and obtains its learning signal
from the final-answer verifier. Development worlds support model selection;
the reported L4 test results use three held-out worlds. World identifiers
such as $w_0$ are local to their split, so a training $w_0$ and a development
$w_0$ do not denote the same corpus.

\subsection{Rendering Conditions}

Each rendering prompt specifies the objects, their colors and marks, and
their spatial arrangement. Reference images maintain an object's identity
or a mark across views. A portrait prompt requests the relevant object
alone. The question text is generated separately and does not reveal the
intermediate matching values.

The symbolic labels describe the intended scene. For these labels to remain
valid in the images, rendering must preserve object identity across views,
distinguish different instances where required, reproduce colors and marks,
and show the stated relations unambiguously. An image audit checks these
properties against the manifest; images with visible violations require
regeneration. These conditions are also the assumptions needed to transfer
the symbolic guarantees in Appendix~\ref{app:proofs} to pixels. The verifier
does not, by itself, certify rendering fidelity or rule out every possible
captioning shortcut.

\section{Symbolic Guarantees}
\label{app:proofs}

This section states what the generator guarantees about its symbolic worlds,
and the assumptions needed for the retrieval bounds. A unique symbolic answer
does not by itself guarantee that a rendered image is unambiguous. Applying
these statements to pixels also requires the rendering conditions in
Appendix~\ref{app:gen}.

\paragraph{Definitions.}
An object's \emph{link key} is its identity, logo, or printed string, depending
on the matching rule in the instruction. A scene contains an \emph{anchor}
that matches the preceding object and a \emph{witness} that supplies the next
key. The witness must satisfy the stated spatial relation and any color
constraint. A portrait contains the final object alone. Let $H$ be the total
number of retrieval hops, including the final portrait; the three-hop task
therefore has two scene hops and one portrait hop.

For a query $Q$ and pooled corpus $D$, let $\mathrm{Sat}(Q,D)$ be the set of
image sequences that satisfy every matching rule, relation, and filter, and
end in an eligible portrait. The acceptable answer set $\mathcal A(Q,D)$
contains the last image of each such sequence. Precise queries are verified
to have one satisfying sequence. For vague queries, the verifier enumerates
the sequences allowed by the omitted relation and accepts all their endpoints.

\begin{theorem}[Whole-corpus uniqueness]
\label{thm:unique}
Suppose each chain family has fresh link keys, every matched scene has a
unique anchor and witness, each non-forked gold prefix has one valid
continuation, and each alternative branch at a fork has no satisfying
completion. If the terminal key has exactly one eligible portrait, then a
precise query has exactly one satisfying trajectory in the pooled corpus.
\end{theorem}

\begin{proof}
The constructed gold trajectory satisfies the query. Consider a different
satisfying trajectory and its first departure from the gold trajectory.
Fresh keys prevent a continuation through another family. At a non-forked
prefix, continuation uniqueness rules out the departure. At a fork, the
alternative has no satisfying completion. A change in the final image is
excluded by terminal-portrait uniqueness. Thus no different satisfying
trajectory exists. The verifier checks these conditions over the pooled
corpus, including decoys and appearance twins.
\end{proof}

\begin{theorem}[Reusing appearance templates]
\label{thm:reuse}
Reusing an object's appearance template across families preserves symbolic
uniqueness if link keys remain fresh and the conditions of
Theorem~\ref{thm:unique} remain satisfied.
\end{theorem}

\begin{proof}
Symbolic matching depends on link keys and scene structure, not on the
appearance template used to render an object. Reusing a template therefore
does not add an object to a key's matching set. For an identity link,
different instances must still be distinguishable in the rendered images.
\end{proof}

\paragraph{Assumptions for direct retrieval.}
A pairwise scorer assigns $f(Q,z)$ to each candidate image $z$ using only
the query and that image. It has no retrieved intermediate evidence,
corpus-level features, or metadata that reveals an image's role. Consider
$n$ candidate portraits that agree on all answer properties supplied by the
query, such as color and matching rule. The required symmetry assumption is
that, conditioned on the evidence available to this scorer, each of these
portraits is equally likely to be the answer. This requires role-independent
sampling of keys and appearances and symmetric handling of collisions.
Omitting a key from the query alone is insufficient to establish this symmetry.

\begin{theorem}[Direct pairwise answer bound]
\label{thm:pairwise}
Under the conditional symmetry above, if $n\ge m+1$ portraits share the
answer's query-specified properties and at least one scene hop precedes the
answer, any pairwise scorer has top-1 answer accuracy at most $1/(m+1)$,
averaged over the sampled worlds.
\end{theorem}

\begin{proof}
Conditioned on the scorer's available evidence, the answer index is uniform
over the $n$ portraits. Any selected index has probability $1/n$ of being
correct. Selecting an image outside this set cannot increase success.
Averaging over the evidence gives the bound. A solver that retrieves an
intermediate image and reads its next key has additional evidence, so this
bound does not apply to it.
\end{proof}

\begin{proposition}[Completing a chain with a tight retrieval budget]
\label{prop:budget}
Consider a policy that follows the query's local transitions and must
retrieve a complete $H$-hop chain using at most $H$ image retrievals. At each
fork $t$, suppose there are $d_t$ incorrect branches and, conditioned on a
correct prefix and the policy's available evidence, the correct branch is
uniform among the $1+d_t$ choices. If an incorrect branch cannot complete
the query, the probability of completing the chain is at most
$\prod_{t:\,d_t>0}(1+d_t)^{-1}$.
\end{proposition}

\begin{proof}
The budget leaves no retrievals for testing a wrong branch and returning to
the correct one. The conditional probability of choosing correctly at fork
$t$ is at most $(1+d_t)^{-1}$. Applying the chain rule over successive forks
gives the product. Unconditional independence between forks is not required.
\end{proof}

This proposition concerns complete-chain retrieval under the stated local
information restriction. It does not bound finding the answer anywhere in
a returned set, guessing the final portrait directly, or using a larger
budget for recovery. Its budget counts image retrievals; the experimental
decision budget also counts \textsc{Backtrack} and \textsc{Stop}. The
benefit of additional search is measured empirically.

\section{Training and Evaluation Details}
\label{app:train}

\subsection{Data, Models, and Evaluation}

All four training configurations in Table~\ref{tab:main} use L4 data and
Qwen3-VL-Embedding-2B~\citep{li2026qwen3}. They vary the prompt type (precise or vague) and
training hop counts (three hops only or a mixture of one, two, and three).
The fully labeled seed world contains $978$ questions and $4{,}492$ corpus
images. SFT and Online~IL use its gold trajectories. RLVR uses separate
practice worlds and a final-answer verifier; it does not train on their
intermediate gold actions. The detailed optimization settings below describe
the precise three-hop reference run, which produces
\texttt{rloo10w10\_200}.

For the three-hop results, the precise-prompt sample sizes at L1--L7 are
$(99,98,97,300,100,94,95)$; the vague-prompt sizes at L1--L5 are
$(99,49,48,150,50)$. L4 pools three held-out test worlds, while the other
levels use development world $w_0$. Development scores are used for model
selection and should not be read as independently held-out estimates.
Embedding and caption baselines measure whether an acceptable answer is
among their five returned images. For a standalone policy, correctness requires \textsc{Stop}
with an acceptable image at the top of the active stack. Agent accuracy
checks the returned answer, including the evaluator's fallback if there is
no valid final declaration. Finding an answer among retrieved images is a
separate metric, defined for the human-authored evaluation in
Appendix~\ref{app:human-authored}.

The active stack $\sigma_t$ contains the query image and the retrieved
images that have not been removed by backtracking. Its encoding is
$s_t=E_s(x,\sigma_t)$. The rollout $\rho$ records every decision, including
selections later undone. Candidate masks and stored parent scores are
maintained by the search procedure. Thus a popped image leaves subsequent
state-encoder inputs but its earlier decision and training record remain
in the rollout. This distinction applies in both Online~IL and RLVR.

\subsection{Supervised Fine-Tuning}

\begin{wrapfigure}{r}{0.3\linewidth}
    \centering
    \includegraphics[width=\linewidth]{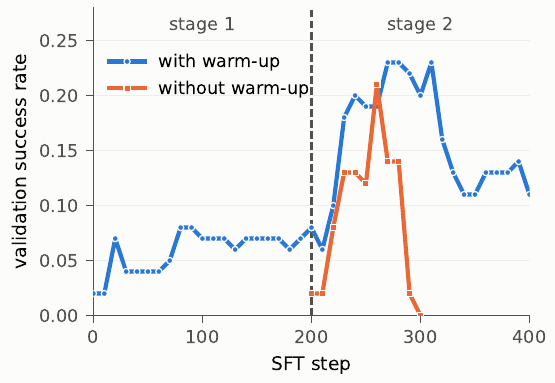}
    \caption{\textbf{State-encoder warm-up stabilizes SFT.}}
    \label{fig:sft}
\end{wrapfigure}

The state and action encoders use separate LoRA adapters over a frozen
backbone, with rank $32$ and scaling parameter $32$. The adapters modify
the query, key, and value projections and the MLP projections. Embeddings
use last-token pooling and are normalized to unit length, with dimension
$2048$. The reference SFT recipe uses learning rate $2\times10^{-5}$,
batch size $8$ per GPU, and a linear schedule over four epochs.

The first stage trains only the state encoder on gold prefixes with
full-corpus InfoNCE at temperature $0.05$. The next stage introduces
hop-level hard negatives, including designed dead ends and constraint
violations, and also updates the action encoder. Hard-negative images are
re-encoded with gradients and their scores replace the corresponding
entries in the cached corpus logits. These re-encoded images provide the
action encoder's gradient; the rest of the cached index is treated as
constant during an update. Later stages freeze the action encoder and its
corpus index.

\subsection{Online Imitation Learning}
\label{app:online-il}

\paragraph{Rollout collection.}
Online~IL runs for $120$ iterations with a $14$-decision budget per rollout.
The oracle selects
the next gold image from a valid prefix, backtracks from an incorrect prefix,
and stops at the completed chain. At iteration $i$, the expert-control and
forced-excursion probabilities are
\begin{equation}
    \beta_i=\max\!\left(0,1-\frac{i}{60}\right),\qquad
    p_i=0.6\max\!\left(0,1-\frac{i}{80}\right).
    \label{eq:online_il_schedules}
\end{equation}
At a valid prefix with an available dead end, the rollout enters a dead end
with probability $p_i$. Otherwise, it follows the oracle with probability
$\beta_i$ and the current policy with probability $1-\beta_i$. Every visited
state receives an oracle label, regardless of how its action was chosen.
Each iteration trains only on the states collected in that iteration;
states from earlier iterations are not retained for training.
Each record stores the stack, candidate mask, parent maximum score, and
oracle action. During the update, we re-encode the recorded stack and
recompute the current retrieval scores and $c_t$, retaining the recorded
mask and parent score. A later backtrack does not remove earlier records.

\paragraph{Online~IL losses and gradients.}
Selection uses the top-$8$ eligible candidates, with the expert's next gold
image appended if absent, and a softmax temperature of $0.05$. The gate
losses in Eq.~\ref{eq:online_il_losses} are
\begin{align}
    B_t(y)&=-w_+y\log p_t^{\mathrm{bt}}
             -(1-y)\log(1-p_t^{\mathrm{bt}}),\nonumber\\
    S_t(y)&=-y\log p_t^{\mathrm{stop}}
             -(1-y)\log(1-p_t^{\mathrm{stop}}),
    \label{eq:online_il_gate_losses}
\end{align}
where $w_+=4$ compensates for the relative scarcity of backtrack targets.
The gate probabilities use Eq.~(\ref{eq:router_gates}) with $\tau_g=1$.
We set $B_t(y)=0$ at the root and $S_t(y)=0$ wherever stopping is disabled.
For the procedural evaluations, three-hop-only checkpoints permit stopping
once the stack depth reaches the question's hop count; mixed-hop checkpoints
permit it after any retrieval. This depth floor is an evaluation constraint
and should be kept distinct from the learned stop decision.
The per-state losses are averaged over the collected dataset:
\begin{equation}
    \mathcal L_{\mathrm{Online\,IL}}
    =\frac{1}{|\mathcal D_i|}\sum_{t=1}^{|\mathcal D_i|}\ell_t.
    \label{eq:online_il_loss}
\end{equation}
The action encoder and its corpus index remain frozen. Selection gradients
reach the state encoder through the retrieval scores, while gate gradients
reach both the corresponding MLP and the state encoder through $s_t$.
The state embedding is not detached on backtrack or stop examples. The four
retrieval-context features are treated as constants: the current maximum
score, the top-1/top-2 gap, the parent maximum, and their difference.

\subsection{Outcome-Based RL Details}
\label{app:rl}

\begin{figure}[H]
\centering
\begin{subfigure}[t]{0.32\textwidth}
    \centering
    \includegraphics[width=\linewidth]{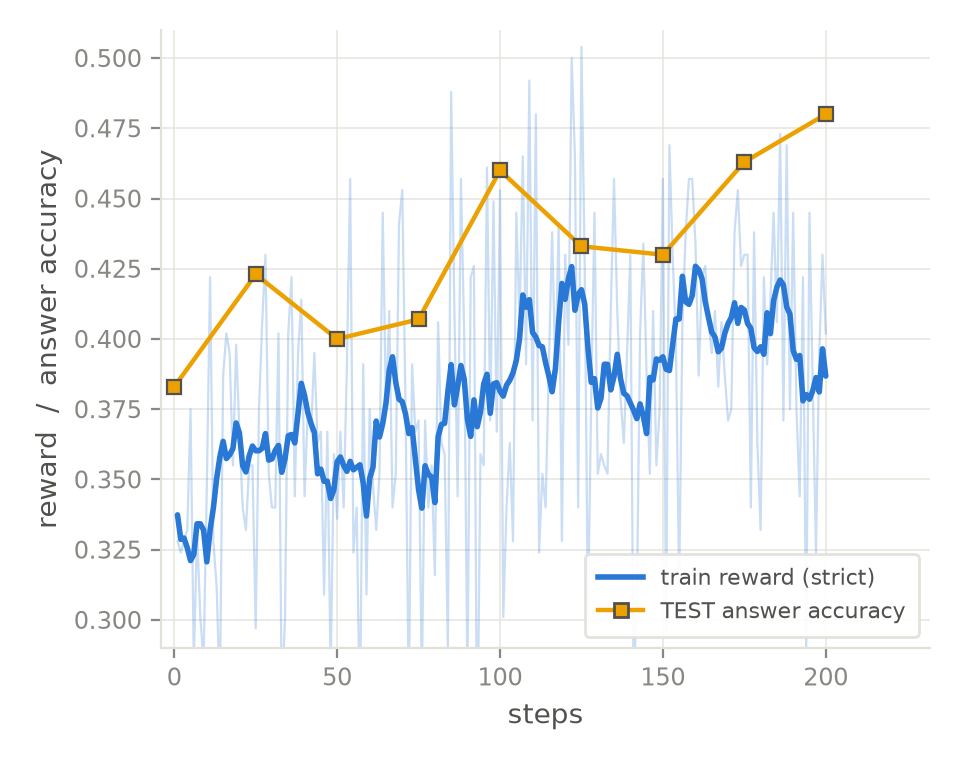}
    \caption{Reward and test accuracy}
    \label{fig:panel_a}
\end{subfigure}\hfill
\begin{subfigure}[t]{0.32\textwidth}
    \centering
    \includegraphics[width=\linewidth]{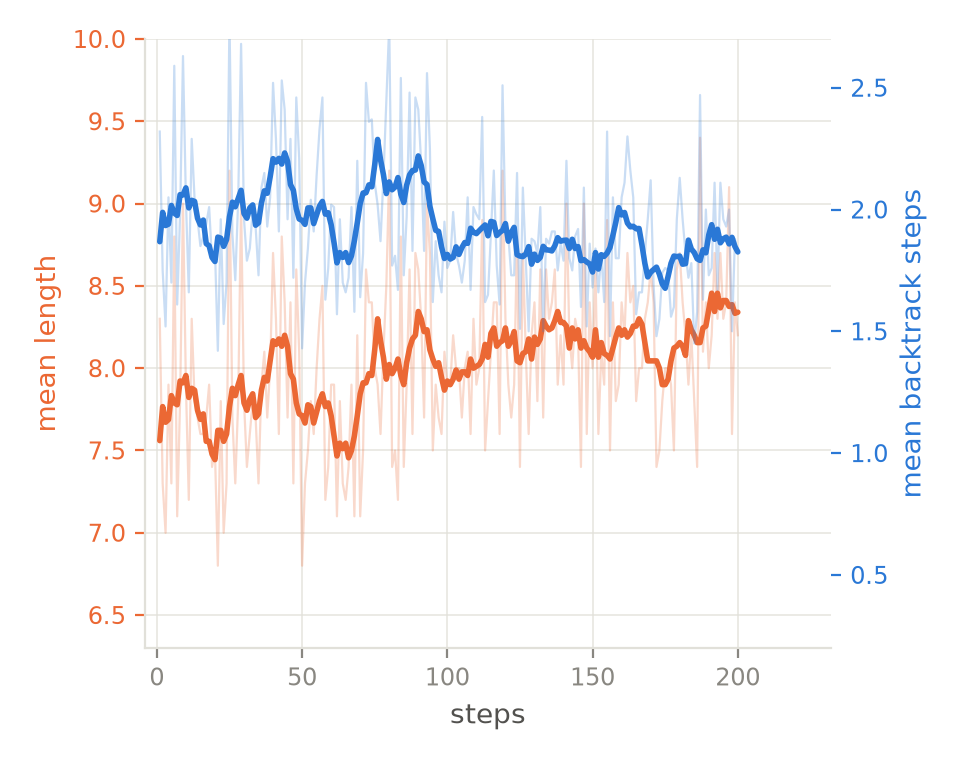}
    \caption{Length of $\rho$ and backtracks}
    \label{fig:panel_b}
\end{subfigure}\hfill
\begin{subfigure}[t]{0.32\textwidth}
    \centering
    \includegraphics[width=\linewidth]{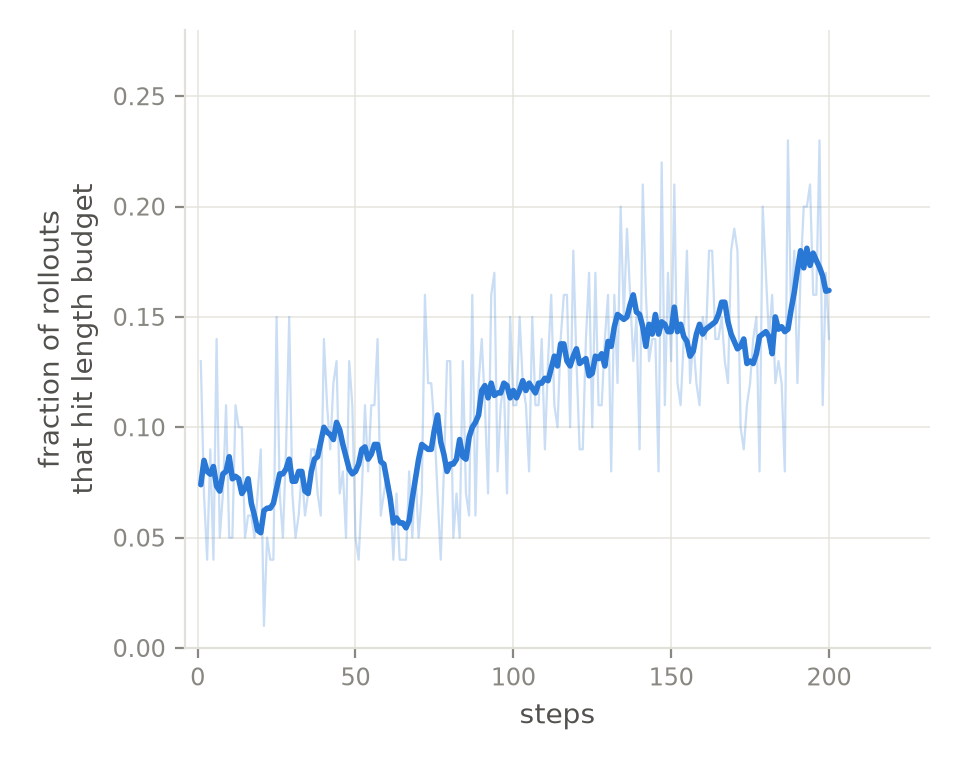}
    \caption{Truncation rate}
    \label{fig:panel_c}
\end{subfigure}
\caption{\textbf{Training dynamics of RLVR.}
(a) Training reward and test accuracy.
(b) Rollout length and number of backtracks. (c) The fraction of rollouts that exceed budget.}
\label{fig:three_panels}
\end{figure}

\paragraph{Rewards and advantages.}
For each input $(x,q_0)$, we sample $G$ rollouts $\rho_g$ from the behavior
policy $\mu$, a snapshot of the current policy at the start of the iteration.
Let $\mathcal A$ be the input's certified answer set: a singleton for a
precise question, or the set of valid answers for a vague question. With
$T_g$ denoting the final decision index, the terminal verifier returns
\begin{equation}
    r_g=\mathbf 1\!\left[
        a_{g,T_g}=\textsc{Stop}\ \land\
        \operatorname{top}(\sigma_{g,T_g})\in\mathcal A
    \right].
    \label{eq:terminal_reward}
\end{equation}
Each decision in $\rho_g$ receives the same leave-one-out advantage
\begin{equation}
    A_g=r_g-\frac{1}{G-1}\sum_{g'\ne g}r_{g'},
    \label{eq:rloo_advantage}
\end{equation}
where the sum is over the other rollouts for the same input.
Groups with identical rewards have zero advantage and contribute no
policy-gradient signal.

\begin{algorithm}[ht]
\caption{RLVR training}
\label{alg:rl}
\small
\begin{algorithmic}[1]
\Require Online~IL policy $\pi_\theta$; terminal answer verifier; frozen action encoder
\State freeze a reference copy $\pi_0$ of the initial policy
\For{each iteration}
    \State snapshot the current policy as the behavior policy $\mu$
    \For{each sampled query}
        \State generate $G$ complete rollouts $\rho_g$ using $\mu$
        \State for every prefix $\rho_{g,<t}$, record its active stack $\sigma_{g,t}$,
        masks, context, and behavior probabilities
        \State compute terminal rewards $r_g$ and leave-one-out advantages $A_g$
    \EndFor
    \For{each inner update epoch}
        \State re-encode retained stacks as $s_{g,t}=E_s(x,\sigma_{g,t})$;
        recompute candidate scores and gate probabilities
        \State update $E_s$ and the gates with clipped branch terms and KL penalties
    \EndFor
\EndFor
\end{algorithmic}
\end{algorithm}

\begin{table}[H]
\centering
\caption{\textbf{Basic RL policy-gradient form.}}
\label{tab:rl_losses}
\small
\setlength{\tabcolsep}{4pt}
\renewcommand{\arraystretch}{1.25}
\begin{tabularx}{0.65\linewidth}{@{}l>{\raggedright\arraybackslash}X@{}}
\toprule
Sampled action $a_t$ & Basic term $\ell_t^{\mathrm{PG}}$ \\
\midrule
\textsc{Select}$(a)$ &
$-A_g\log\pi_\theta^{\mathrm{sel}}(a\mid s_t,\mathcal C_t)$ \\
\textsc{Backtrack} & $-A_g\log p_t^{\mathrm{bt}}$ \\
\textsc{Stop} & $-A_g\log p_t^{\mathrm{stop}}$ \\
\bottomrule
\end{tabularx}
\end{table}

\paragraph{Branch-local credit assignment.}
For decision $t$ in rollout $\rho_g$, let $p_{g,t}$ denote the sampled-branch
probability factor in Table~\ref{tab:rl_losses}, and $\bar p_{g,t}$ its stored
value under $\mu$. The SELECT and gate probabilities follow
Eqs.~(\ref{eq:select_policy}) and~(\ref{eq:router_gates}). We omit the
complementary gate factors in Eq.~(\ref{eq:router_policy}):
$(1-p_t^{\mathrm{stop}})(1-p_t^{\mathrm{bt}})$ for SELECT and
$(1-p_t^{\mathrm{stop}})$ for BACKTRACK. This generally yields a biased
surrogate for the full policy gradient obtained from
Eq.~(\ref{eq:router_traj}). This choice prevents a SELECT policy-loss
term from directly changing the gates through their rejection probabilities.

The branch terms share one objective and optimizer. SELECT gradients reach
$E_s$; BACKTRACK and STOP gradients reach $E_s$ and the respective gate MLP.
The policy loss for SELECT has no direct gate-MLP gradient. Shared-encoder
updates and KL penalties can still change other branch probabilities;
$E_a$ remains fixed.

\paragraph{DAPO-style clipping and negative-advantage weighting.}
For batch reuse, we use a PPO-style ratio surrogate~\citep{schulman2017proximal}
with DAPO's asymmetric clipping bounds~\citep{yu2026dapo}:
\begin{align}
    \omega_{g,t}(p)&=\frac{p}{\bar p_{g,t}},
    \label{eq:branch_ratio}\\
    \mathcal J_{g,t}(p)&=-w(A_g)\min\!\left\{
        \omega_{g,t}(p)A_g,\,
        \operatorname{clip}\!\left(\omega_{g,t}(p),
        1-\epsilon_{\mathrm{lo}},1+\epsilon_{\mathrm{hi}}\right)A_g
    \right\},
    \label{eq:rl_clipped_loss}
\end{align}
with
\begin{equation}
    w(A)=
    \begin{cases}
        1, & A\ge0,\\
        \lambda_{\mathrm{neg}}, & A<0.
    \end{cases}
    \label{eq:advantage_weight}
\end{equation}
Setting $\epsilon_{\mathrm{hi}}>\epsilon_{\mathrm{lo}}$ permits larger
probability increases for positive-advantage decisions before clipping.
The additional weight $0<\lambda_{\mathrm{neg}}<1$ reduces the magnitude of
negative-advantage terms without reversing their sign. At $p=\bar p_{g,t}$,
with $w=1$ and fixed probability support, the surrogate's gradient matches
the $-A_g\log p$ form in Table~\ref{tab:rl_losses}.

The implemented policy objective sums the clipped terms:
\begin{equation}
    \mathcal L_{\mathrm{PG}}
    =\frac{1}{N_{\mathrm{roll}}}
      \sum_{\rho_g\in\mathcal B}\sum_{t=0}^{T_g}
      \mathcal J_{g,t}(p_{g,t}).
    \label{eq:rloo_loss}
\end{equation}
Here, $\mathcal B$ contains the nonzero-advantage rollouts, while
$N_{\mathrm{roll}}$ counts all sampled rollouts, including zero-advantage
groups. States excluded by the training filters below contribute zero
policy loss. During batch reuse, a SELECT term is evaluated only while the
sampled image remains in the current top-$K$ set $\mathcal C_t$.

\paragraph{RL regularization.}
The full objective adds behavior and reference KL penalties:
\begin{equation}
    \mathcal L_{\mathrm{RL}}=\mathcal L_{\mathrm{PG}}
        +\beta_{\mathrm{beh}}\mathcal K_{\mathrm{beh}}
        +\beta_{\mathrm{ref}}(i)\mathcal K_{\mathrm{ref}}.
    \label{eq:total_rl_loss}
\end{equation}
The behavior policy $\mu$ is the policy that generated the current rollout
batch. We record its selection probabilities before updating the model.
The behavior penalty sums
$\mathrm{KL}(\pi_\theta^{\mathrm{sel}}\,\|\,\mu^{\mathrm{sel}})$
over retained SELECT states and divides by $N_{\mathrm{roll}}$.
Both distributions use the candidate set stored during rollout.
The stored probabilities remain fixed through all updates on this batch;
they are refreshed only when the updated policy collects the next batch.
Thus, behavior KL limits changes from the policy at the start of the
current iteration, even when the batch is reused for several gradient updates.
The reference policy $\pi_0$ is a frozen copy of the initial Online~IL policy.
On a subsample of retained states, $\mathcal K_{\mathrm{ref}}$ adds the
SELECT KL on the stored candidate set when selection was taken, and
Bernoulli KLs for each permitted gate. All KLs use the current policy as
their first argument. The reference sum is normalized by the number of
reference states, floored at $16$. Its weight decreases during training
to stabilize early updates while allowing later improvement without
gold-trajectory supervision.

\paragraph{RL configuration.}
The precise three-hop reference run uses eight data-parallel GPUs. At each iteration, each rank samples
$4$ practice questions and $G=8$ stochastic rollouts per question, each
limited to $16$ decisions. The selection temperature is $\tau=0.05$ and
the gate temperature is $\tau_g=0.5$. Each group stays on one rank, so
leave-one-out baselines require no communication. The state-encoder LoRA
learning rate is $5\times10^{-6}$ and the gate-MLP rate is $3\times10^{-4}$.
The action encoder and corpus index remain frozen throughout.
We reuse each batch for two inner epochs with DAPO-style asymmetric clipping
($\epsilon_{\mathrm{lo}}=0.2$, $\epsilon_{\mathrm{hi}}=0.3$) and
$\lambda_{\mathrm{neg}}=0.5$.

Gradient-bearing states are restricted to $|\sigma_t|\le9$, including the
query image (equivalently, chain depth $d_t\le8$), to bound activation
memory. If more than $128$ eligible states remain per rank, we sample
$128$ uniformly and scale the policy loss by the original count divided
by $128$. The behavior-KL weight is $\beta_{\mathrm{beh}}=0.05$.
Reference KL is evaluated on at most $48$ retained states per rank, with
$\beta_{\mathrm{ref}}(i)=0.3$ for iterations $1$--$75$ and $0.1$ thereafter.
The reference run lasts $200$ iterations ($\approx14$ hours) and yields
checkpoint \texttt{rloo10w10\_200}.

\paragraph{Length-penalty variant.}
This variant keeps the terminal verifier unchanged but substitutes
$\tilde r_g=r_g-\lambda n_{\mathrm{decisions}}(\rho_g)$
for $r_g$ in Eq.~(\ref{eq:rloo_advantage}).
The decision count includes SELECT, BACKTRACK, and STOP; strict reward
remains the evaluation and checkpoint-selection criterion.
The reference run uses $\lambda=0$.
The recorded $\lambda=0.02$ run uses one-state gradient microbatches and
raises the gradient-bearing stack limit to $12$ images to accommodate
the additional failed-rollout states with nonzero advantages. This
implementation difference should be considered alongside the reward change.
Appendix~\ref{sec:appendix:lenpen} examines the effect of charging for each decision.

\paragraph{Roles of the optimization components.}
The leave-one-out baseline compares outcomes for the same query without
normalizing by a small within-group reward standard deviation. Asymmetric
clipping limits changes during batch reuse, while the negative-advantage
weight reduces the contribution of failed rollouts. Behavior KL constrains
updates relative to the policy that collected the batch. Reference KL
limits drift from the initial Online~IL policy.

\section{Additional Results and Ablations}
\label{app:results}

\subsection{One- and Two-Hop Queries}
\label{app:full-ladder}

Tables~\ref{tab:main-h1} and~\ref{tab:main-h2} report the complete results
on verified suffixes of the three-hop queries. The corpus and checkpoints
are unchanged. The three-hop-only policies retain the question-specific
minimum stop depth; mixed-hop policies can stop after any retrieval.
As in the main table, embedding baselines receive credit if the answer is
among five results, whereas the standalone policy must stop on one correct
answer. This distinction is especially relevant on one-hop queries, where
a top-5 list contains five possible answers.

\begin{table}[H]
\centering
\caption{\textbf{One-hop success rate (\%).}
Same corpora, checkpoints, and scoring conventions as Table~\ref{tab:main}.
L4 uses test worlds $w_1$--$w_3$; other levels use development world $w_0$.
Caption$\to$Qwen3 retains the query image as pixels. In block~(b), the
agent uses the learned per-step retriever; the other trained-agent blocks
use the chain tool. Vague prompts are evaluated on L1--L5.
Bold marks the best result within each training block on each side of the
dashed rule.}
\label{tab:main-h1}
\scriptsize
\setlength{\tabcolsep}{1.1pt}
\renewcommand{\arraystretch}{1.1}
\resizebox{\textwidth}{!}{%
\begin{tabular}{@{}lr *{12}{c} c !{\dvrule} cc c@{}}
\toprule
& & \multicolumn{13}{c}{Retrieval methods} & \multicolumn{3}{!{\dvrule}c}{Agentic search} \\
\cmidrule(lr){3-15}\cmidrule(lr){16-18}
& & \multicolumn{6}{c}{Embedding retrieval} & \multicolumn{6}{c}{Caption $\rightarrow$ embedding}
& \router & GE2 & Qwen3 & $+$\,\router \\
\cmidrule(lr){3-8}\cmidrule(lr){9-14}
& & \multicolumn{3}{c}{GE2} & \multicolumn{3}{c}{Qwen3} & \multicolumn{3}{c}{GE2} & \multicolumn{3}{c}{Qwen3}
& & & & \\
\cmidrule(lr){3-5}\cmidrule(lr){6-8}\cmidrule(lr){9-11}\cmidrule(lr){12-14}
Level & $n$ & 1-shot & Greedy & Hist. & 1-shot & Greedy & Hist.
& 1-shot & Greedy & Hist. & 1-shot & Greedy & Hist. & & & & \\
\midrule
\multicolumn{18}{@{}l}{\textbf{(a) Precise prompt, $3$-hop training}}\\
\midrule
L1 (re-ID) & 99 & 90.9 & 77.8 & 78.8 & \textbf{98.0} & 92.9 & 96.0 & 27.3 & 16.2 & 18.2 & 92.9 & 73.7 & 73.7 & 69.7 & 80.8 & 74.7 & \textbf{96.0} \\
L2 (spatial) & 98 & 92.9 & 77.6 & 81.6 & \textbf{93.9} & 89.8 & 90.8 & 27.6 & 10.2 & 14.3 & 86.7 & 64.3 & 70.4 & 67.3 & 81.6 & 78.6 & \textbf{95.9} \\
L3 (color) & 97 & 93.8 & 79.4 & 83.5 & \textbf{96.9} & 89.7 & 93.8 & 14.4 & 4.1 & 7.2 & 84.5 & 66.0 & 70.1 & 74.2 & 85.6 & 82.5 & \textbf{95.9} \\
L4 (dead ends) & 300 & 90.7 & 70.0 & 73.7 & \textbf{97.3} & 88.0 & 93.7 & 14.7 & 8.0 & 9.7 & 86.7 & 64.0 & 65.3 & 67.3 & 77.7 & 75.7 & \textbf{96.7} \\
L5 (twins) & 100 & 89.0 & 67.0 & 76.0 & \textbf{98.0} & 88.0 & 95.0 & 14.0 & 6.0 & 8.0 & 86.0 & 66.0 & 73.0 & 59.0 & 78.0 & 66.0 & \textbf{94.0} \\
L6 (logo) & 94 & \textbf{57.4} & 42.6 & 47.9 & 55.3 & 47.9 & 50.0 & 6.4 & 5.3 & 4.3 & 56.4 & 40.4 & 43.6 & 28.7 & 58.5 & 62.8 & \textbf{80.9} \\
L7 (text) & 95 & \textbf{49.5} & 37.9 & 36.8 & 41.1 & 31.6 & 33.7 & 5.3 & 1.1 & 2.1 & \textbf{49.5} & 27.4 & 29.5 & 12.6 & 73.7 & 75.8 & \textbf{80.0} \\
\midrule
\multicolumn{18}{@{}l}{\textbf{(b) Precise prompt, mixed-hop training}}\\
\midrule
L1 (re-ID) & 99 & 90.9 & 77.8 & 78.8 & \textbf{98.0} & 92.9 & 96.0 & 27.3 & 16.2 & 18.2 & 92.9 & 73.7 & 73.7 & \textbf{98.0} & \textbf{80.8} & 74.7 & 69.7 \\
L2 (spatial) & 98 & 92.9 & 77.6 & 81.6 & 93.9 & 89.8 & 90.8 & 27.6 & 10.2 & 14.3 & 86.7 & 64.3 & 70.4 & \textbf{96.9} & \textbf{81.6} & 78.6 & 68.4 \\
L3 (color) & 97 & 93.8 & 79.4 & 83.5 & 96.9 & 89.7 & 93.8 & 14.4 & 4.1 & 7.2 & 84.5 & 66.0 & 70.1 & \textbf{97.9} & \textbf{85.6} & 82.5 & 78.4 \\
L4 (dead ends) & 300 & 90.7 & 70.0 & 73.7 & 97.3 & 88.0 & 93.7 & 14.7 & 8.0 & 9.7 & 86.7 & 64.0 & 65.3 & \textbf{99.7} & \textbf{77.7} & 75.7 & 76.7 \\
L5 (twins) & 100 & 89.0 & 67.0 & 76.0 & \textbf{98.0} & 88.0 & 95.0 & 14.0 & 6.0 & 8.0 & 86.0 & 66.0 & 73.0 & 96.0 & \textbf{78.0} & 66.0 & 74.0 \\
L6 (logo) & 94 & \textbf{57.4} & 42.6 & 47.9 & 55.3 & 47.9 & 50.0 & 6.4 & 5.3 & 4.3 & 56.4 & 40.4 & 43.6 & 50.0 & 58.5 & \textbf{62.8} & 59.6 \\
L7 (text) & 95 & \textbf{49.5} & 37.9 & 36.8 & 41.1 & 31.6 & 33.7 & 5.3 & 1.1 & 2.1 & \textbf{49.5} & 27.4 & 29.5 & 27.4 & 73.7 & \textbf{75.8} & 74.7 \\
\midrule
\multicolumn{18}{@{}l}{\textbf{(c) Vague prompt, $3$-hop training}}\\
\midrule
L1 (re-ID) & 99 & 46.5 & 21.2 & 27.3 & \textbf{73.7} & 45.5 & 55.6 & 24.2 & 9.1 & 10.1 & 24.2 & 13.1 & 15.2 & 47.5 & 15.2 & 14.1 & \textbf{79.8} \\
L2 (spatial) & 98 & 56.1 & 30.6 & 40.8 & \textbf{76.5} & 59.2 & 64.3 & 17.3 & 9.2 & 10.2 & 25.5 & 13.3 & 13.3 & 45.9 & 32.7 & 20.4 & \textbf{90.8} \\
L3 (color) & 96 & 55.2 & 28.1 & 35.4 & \textbf{65.6} & 47.9 & 55.2 & 11.5 & 7.3 & 7.3 & 22.9 & 10.4 & 12.5 & 33.3 & 19.8 & 15.6 & \textbf{79.2} \\
L4 (dead ends) & 300 & 47.0 & 19.3 & 26.3 & \textbf{66.3} & 41.3 & 47.7 & 13.3 & 6.7 & 7.7 & 18.0 & 7.7 & 9.3 & 26.3 & 21.0 & 21.3 & \textbf{77.0} \\
L5 (twins) & 100 & 34.0 & 14.0 & 19.0 & \textbf{58.0} & 31.0 & 40.0 & 10.0 & 6.0 & 6.0 & 13.0 & 7.0 & 5.0 & 31.0 & 27.0 & 20.0 & \textbf{82.0} \\
\midrule
\multicolumn{18}{@{}l}{\textbf{(d) Vague prompt, mixed-hop training}}\\
\midrule
L1 (re-ID) & 99 & 46.5 & 21.2 & 27.3 & 73.7 & 45.5 & 55.6 & 24.2 & 9.1 & 10.1 & 24.2 & 13.1 & 15.2 & \textbf{87.9} & 15.2 & 14.1 & \textbf{97.0} \\
L2 (spatial) & 98 & 56.1 & 30.6 & 40.8 & 76.5 & 59.2 & 64.3 & 17.3 & 9.2 & 10.2 & 25.5 & 13.3 & 13.3 & \textbf{93.9} & 32.7 & 20.4 & \textbf{100.0} \\
L3 (color) & 96 & 55.2 & 28.1 & 35.4 & 65.6 & 47.9 & 55.2 & 11.5 & 7.3 & 7.3 & 22.9 & 10.4 & 12.5 & \textbf{92.7} & 19.8 & 15.6 & \textbf{96.9} \\
L4 (dead ends) & 300 & 47.0 & 19.3 & 26.3 & 66.3 & 41.3 & 47.7 & 13.3 & 6.7 & 7.7 & 18.0 & 7.7 & 9.3 & \textbf{89.7} & 21.0 & 21.3 & \textbf{97.0} \\
L5 (twins) & 100 & 34.0 & 14.0 & 19.0 & 58.0 & 31.0 & 40.0 & 10.0 & 6.0 & 6.0 & 13.0 & 7.0 & 5.0 & \textbf{84.0} & 27.0 & 20.0 & \textbf{95.0} \\
\bottomrule
\end{tabular}%
}
\end{table}

\clearpage

\begin{table}[H]
\centering
\caption{\textbf{Two-hop success rate (\%).}
Same corpora, checkpoints, and scoring conventions as Table~\ref{tab:main}.
L4 uses test worlds $w_1$--$w_3$; other levels use development world $w_0$.
Caption$\to$Qwen3 retains the query image as pixels. In block~(b), the
agent uses the learned per-step retriever; the other trained-agent blocks
use the chain tool. Vague prompts are evaluated on L1--L5.
Bold marks the best result within each training block on each side of the
dashed rule.}
\label{tab:main-h2}
\scriptsize
\setlength{\tabcolsep}{1.1pt}
\renewcommand{\arraystretch}{1.1}
\resizebox{\textwidth}{!}{%
\begin{tabular}{@{}lr *{12}{c} c !{\dvrule} cc c@{}}
\toprule
& & \multicolumn{13}{c}{Retrieval methods} & \multicolumn{3}{!{\dvrule}c}{Agentic search} \\
\cmidrule(lr){3-15}\cmidrule(lr){16-18}
& & \multicolumn{6}{c}{Embedding retrieval} & \multicolumn{6}{c}{Caption $\rightarrow$ embedding}
& \router & GE2 & Qwen3 & $+$\,\router \\
\cmidrule(lr){3-8}\cmidrule(lr){9-14}
& & \multicolumn{3}{c}{GE2} & \multicolumn{3}{c}{Qwen3} & \multicolumn{3}{c}{GE2} & \multicolumn{3}{c}{Qwen3}
& & & & \\
\cmidrule(lr){3-5}\cmidrule(lr){6-8}\cmidrule(lr){9-11}\cmidrule(lr){12-14}
Level & $n$ & 1-shot & Greedy & Hist. & 1-shot & Greedy & Hist.
& 1-shot & Greedy & Hist. & 1-shot & Greedy & Hist. & & & & \\
\midrule
\multicolumn{18}{@{}l}{\textbf{(a) Precise prompt, $3$-hop training}}\\
\midrule
L1 (re-ID) & 99 & 2.0 & 13.1 & 12.1 & 1.0 & 21.2 & 12.1 & 0.0 & 2.0 & 3.0 & 0.0 & 6.1 & 5.1 & \textbf{81.8} & 90.9 & 83.8 & \textbf{96.0} \\
L2 (spatial) & 98 & 2.0 & 14.3 & 9.2 & 4.1 & 20.4 & 12.2 & 0.0 & 1.0 & 3.1 & 4.1 & 5.1 & 9.2 & \textbf{44.9} & \textbf{91.8} & 88.8 & 72.4 \\
L3 (color) & 97 & 4.1 & 13.4 & 10.3 & 2.1 & 14.4 & 5.2 & 0.0 & 0.0 & 1.0 & 1.0 & 3.1 & 4.1 & \textbf{46.4} & 84.5 & \textbf{85.6} & 79.4 \\
L4 (dead ends) & 300 & 2.7 & 8.3 & 8.0 & 1.0 & 12.3 & 6.0 & 0.0 & 1.0 & 0.3 & 1.0 & 4.3 & 4.3 & \textbf{45.0} & 65.3 & 66.7 & \textbf{75.0} \\
L5 (twins) & 100 & 2.0 & 8.0 & 6.0 & 2.0 & 6.0 & 4.0 & 0.0 & 0.0 & 0.0 & 0.0 & 3.0 & 4.0 & \textbf{38.0} & 50.0 & 64.0 & \textbf{70.0} \\
L6 (logo) & 94 & 3.2 & 3.2 & 2.1 & 4.3 & 7.4 & 4.3 & 1.1 & 0.0 & 0.0 & 2.1 & 2.1 & 1.1 & \textbf{9.6} & 23.4 & 25.5 & \textbf{28.7} \\
L7 (text) & 95 & 5.3 & \textbf{6.3} & 4.2 & 1.1 & 2.1 & 0.0 & 1.1 & 0.0 & 0.0 & 1.1 & 1.1 & 2.1 & 2.1 & 36.8 & \textbf{43.2} & 29.5 \\
\midrule
\multicolumn{18}{@{}l}{\textbf{(b) Precise prompt, mixed-hop training}}\\
\midrule
L1 (re-ID) & 99 & 2.0 & 13.1 & 12.1 & 1.0 & 21.2 & 12.1 & 0.0 & 2.0 & 3.0 & 0.0 & 6.1 & 5.1 & \textbf{89.9} & \textbf{90.9} & 83.8 & 89.9 \\
L2 (spatial) & 98 & 2.0 & 14.3 & 9.2 & 4.1 & 20.4 & 12.2 & 0.0 & 1.0 & 3.1 & 4.1 & 5.1 & 9.2 & \textbf{92.9} & \textbf{91.8} & 88.8 & 90.8 \\
L3 (color) & 97 & 4.1 & 13.4 & 10.3 & 2.1 & 14.4 & 5.2 & 0.0 & 0.0 & 1.0 & 1.0 & 3.1 & 4.1 & \textbf{94.8} & 84.5 & 85.6 & \textbf{92.8} \\
L4 (dead ends) & 300 & 2.7 & 8.3 & 8.0 & 1.0 & 12.3 & 6.0 & 0.0 & 1.0 & 0.3 & 1.0 & 4.3 & 4.3 & \textbf{90.0} & 65.3 & 66.7 & \textbf{75.3} \\
L5 (twins) & 100 & 2.0 & 8.0 & 6.0 & 2.0 & 6.0 & 4.0 & 0.0 & 0.0 & 0.0 & 0.0 & 3.0 & 4.0 & \textbf{87.0} & 50.0 & 64.0 & \textbf{77.0} \\
L6 (logo) & 94 & 3.2 & 3.2 & 2.1 & 4.3 & 7.4 & 4.3 & 1.1 & 0.0 & 0.0 & 2.1 & 2.1 & 1.1 & \textbf{23.4} & 23.4 & 25.5 & \textbf{26.6} \\
L7 (text) & 95 & 5.3 & \textbf{6.3} & 4.2 & 1.1 & 2.1 & 0.0 & 1.1 & 0.0 & 0.0 & 1.1 & 1.1 & 2.1 & \textbf{6.3} & 36.8 & 43.2 & \textbf{46.3} \\
\midrule
\multicolumn{18}{@{}l}{\textbf{(c) Vague prompt, $3$-hop training}}\\
\midrule
L1 (re-ID) & 99 & 1.0 & 8.1 & 6.1 & 2.0 & 11.1 & 3.0 & 0.0 & 2.0 & 1.0 & 0.0 & 3.0 & 1.0 & \textbf{37.4} & 23.2 & 28.3 & \textbf{79.8} \\
L2 (spatial) & 49 & 0.0 & 4.1 & 4.1 & 0.0 & 14.3 & 0.0 & 2.0 & 8.2 & 4.1 & 0.0 & 4.1 & 2.0 & \textbf{40.8} & 24.5 & 30.6 & \textbf{81.6} \\
L3 (color) & 48 & 0.0 & 6.2 & 8.3 & 0.0 & 14.6 & 6.2 & 0.0 & 0.0 & 0.0 & 0.0 & 0.0 & 0.0 & \textbf{25.0} & 6.2 & 6.2 & \textbf{72.9} \\
L4 (dead ends) & 150 & 0.0 & 2.7 & 1.3 & 0.7 & 7.3 & 4.0 & 0.0 & 0.0 & 0.0 & 0.0 & 1.3 & 0.0 & \textbf{20.0} & 7.3 & 6.7 & \textbf{59.3} \\
L5 (twins) & 50 & 0.0 & 4.0 & 2.0 & 0.0 & 10.0 & 2.0 & 0.0 & 0.0 & 0.0 & 0.0 & 0.0 & 0.0 & \textbf{22.0} & 2.0 & 10.0 & \textbf{58.0} \\
\midrule
\multicolumn{18}{@{}l}{\textbf{(d) Vague prompt, mixed-hop training}}\\
\midrule
L1 (re-ID) & 99 & 1.0 & 8.1 & 6.1 & 2.0 & 11.1 & 3.0 & 0.0 & 2.0 & 1.0 & 0.0 & 3.0 & 1.0 & \textbf{55.6} & 23.2 & 28.3 & \textbf{80.8} \\
L2 (spatial) & 49 & 0.0 & 4.1 & 4.1 & 0.0 & 14.3 & 0.0 & 2.0 & 8.2 & 4.1 & 0.0 & 4.1 & 2.0 & \textbf{59.2} & 24.5 & 30.6 & \textbf{93.9} \\
L3 (color) & 48 & 0.0 & 6.2 & 8.3 & 0.0 & 14.6 & 6.2 & 0.0 & 0.0 & 0.0 & 0.0 & 0.0 & 0.0 & \textbf{54.2} & 6.2 & 6.2 & \textbf{83.3} \\
L4 (dead ends) & 150 & 0.0 & 2.7 & 1.3 & 0.7 & 7.3 & 4.0 & 0.0 & 0.0 & 0.0 & 0.0 & 1.3 & 0.0 & \textbf{53.3} & 7.3 & 6.7 & \textbf{84.7} \\
L5 (twins) & 50 & 0.0 & 4.0 & 2.0 & 0.0 & 10.0 & 2.0 & 0.0 & 0.0 & 0.0 & 0.0 & 0.0 & 0.0 & \textbf{48.0} & 2.0 & 10.0 & \textbf{78.0} \\
\bottomrule
\end{tabular}%
}
\end{table}

\subsection{Agent Model and Retrieval Tool}
\label{app:model_tool}

\paragraph{Settings for Figure~\ref{fig:vhop_results}.}
Success rates use the same $300$ precise three-hop L4 test queries.
Panel (a) compares Qwen3 and \router{}, each standalone and as the
Flash agent's retrieval tool. \router{} uses precise mixed-hop training
and checkpoint \texttt{mixh\_rl\_best}. Colors identify retrieval tools;
circles denote standalone retrieval, and diamonds denote Flash agents.
The green dashed line connects the Pareto frontier among the four
plotted configurations.

The horizontal axis uses a linear scale. It adds generated tokens, including
thinking tokens, to retrieval hops, counting each hop as one token. For
agents, we sum tokens over all model calls and use the length of the
recorded retrieval trajectory as the hop count. These means come from a
separate pilot on $10$ matched queries in $w_1$. Standalone Qwen3 uses
five retrieval steps. The standalone Router count is estimated as $4+2b$,
where $b$ is the mean number of backtracks, giving $10.11$ for the
mixed-hop policy. The corresponding means for Flash with Qwen3 and
\router{} are $1{,}885.8$ and $727.9$. This is a token-count convention,
not a measure of latency or total computation.

Untrained retrievers count an answer found anywhere in the returned
images; \router{} allows $16$ policy actions and requires a correct answer
at \textsc{Stop}. Agents are scored on their final answers. Agents with
GE2 or Qwen3 allow $16$ generation rounds; agents with \router{} allow $8$.
The generation limit is $8{,}000$ tokens per model call.
Panel (b) uses the same mixed-hop Router and Flash + Qwen3 baseline:
upgrading Flash to Pro adds $3.7$ percentage points, while replacing
Qwen3 with \router{} adds $52.7$
(Tables~\ref{tab:main} and~\ref{tab:vlm16}).

Figure~\ref{fig:vlm-full} extends Figure~\ref{fig:vhop_results}b to all seven levels and both tool training settings.
Table~\ref{tab:vlm16} lists all model--tool combinations.
On L4, Flash + Qwen3 achieves $32.7\%$ success. Upgrading Flash to Pro raises
this to $36.3\%$; replacing Qwen3 with the precise three-hop or mixed-hop
\router{} raises it to $59.7\%$ or $85.4\%$, respectively. The gains are
$3.7$, $27.0$, and $52.7$ percentage points, computed before rounding.
With the mixed-hop tool, Flash and Pro reach $85.4\%$ and $84.7\%$.

\begin{table}[H]
\centering
\caption{\textbf{Agent model and retrieval tool.}
Three-hop answer accuracy (\%) with Gemini~3.5~Flash or Gemini~3.1~Pro,
up to $16$ tool calls, and the same answer-scoring convention.
Both Router tools use precise training prompts.
L4 uses the $300$-query test set; other levels use development world $w_0$.
The mean averages the seven levels. Bold marks the highest value in each
column.}
\label{tab:vlm16}
\small
\setlength{\tabcolsep}{3.6pt}
\renewcommand{\arraystretch}{1.05}
\begin{tabular}{@{}llrrrrrrr r@{}}
\toprule
Retrieval tool & Agent & L1 & L2 & L3 & L4 & L5 & L6 & L7 & mean \\
\midrule
\multirow{2}{*}{GE2 (untrained)} & Flash & 60.6 & 58.2 & 56.7 & 19.3 & 15.0 & 6.4 & 12.6 & 32.7 \\
 & Pro & \textbf{93.9} & \textbf{89.8} & \textbf{90.7} & 37.0 & 40.0 & 11.7 & 11.6 & 53.5 \\
\midrule
\multirow{2}{*}{Qwen3 (untrained)} & Flash & 73.7 & 69.4 & 77.3 & 32.7 & 28.0 & 9.6 & 14.7 & 43.6 \\
 & Pro & 78.8 & 77.6 & \textbf{90.7} & 36.3 & 34.0 & 9.6 & \textbf{18.9} & 49.4 \\
\midrule
\multirow{2}{*}{\router, $3$-hop} & Flash & 85.9 & 62.2 & 58.8 & 59.7 & 57.0 & 12.8 & 6.3 & 49.0 \\
 & Pro & 82.8 & 64.3 & 72.2 & 64.7 & 61.0 & \textbf{16.0} & 6.3 & 52.5 \\
\midrule
\multirow{2}{*}{\router, mixed-hop} & Flash & 80.8 & 80.6 & 89.7 & \textbf{85.4} & \textbf{84.0} & 10.6 & 5.3 & \textbf{62.3} \\
 & Pro & 78.8 & 81.6 & 88.7 & 84.7 & 83.0 & 11.7 & 4.2 & 61.8 \\
\bottomrule
\end{tabular}
\end{table}

\begin{figure}[htbp]
\centering
\includegraphics[width=\linewidth]{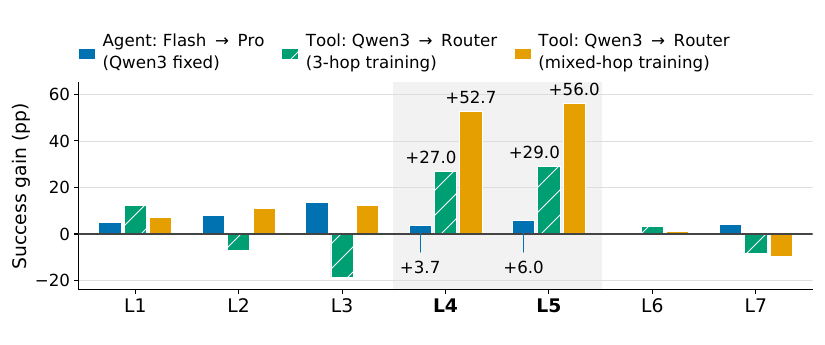}
\caption{\textbf{Agent and tool upgrades across all seven levels.}
Full version of Figure~\ref{fig:vhop_results}b, showing gains over Flash + Qwen3 on
precise three-hop queries, including both Router training configurations.
Router labels indicate training hops.
L4 uses the $300$-query test set; other levels use development world $w_0$.
Gains are computed before rounding.}
\label{fig:vlm-full}
\end{figure}

\subsection{Number of Images Returned per Call}

Table~\ref{tab:topk-full} reports accuracy and context usage for the
retrieval-width experiment. Larger candidate sets improve the untrained
tools, with GE2 reaching $81.3\%$ on L4 at $k=50$. The precise mixed-hop
\router{} reaches $85.4\%$ while returning fewer images. The number of
image occurrences sent to the model includes repeated conversation
history; it is distinct from the number of retrieved images and from
measured latency or monetary cost.

\begin{table}[H]
\centering
\caption{\textbf{Full retrieval-width comparison.}
Three-hop answer accuracy (\%) with Flash and up to $16$ tool calls.
L4 uses test worlds $w_1$--$w_3$; other levels use development world $w_0$.
A chain call returns the policy's active image chain. L4 context statistics
are per-query means: \emph{calls} counts tool calls, \emph{images} counts
returned images including repeats, and \emph{sent} also counts repeated
images in the conversation history. At $k=50$, three GE2 and four Qwen3
conversations reached the request-size limit and used the last top-1
candidate as the fallback answer. Dashes indicate settings not evaluated.}
\label{tab:topk-full}
\small
\setlength{\tabcolsep}{3.1pt}
\renewcommand{\arraystretch}{1.05}
\begin{tabular}{@{}lrrrrrrrr rrr@{}}
\toprule
Retrieval tool & $k$ & L1 & L2 & L3 & L4 & L5 & L6 & L7 & calls & images & sent \\
\midrule
\multirow{5}{*}{GE2 (untrained)} & 1 & 60.6 & 58.2 & 56.7 & 19.3 & 15.0 & 6.4 & 12.6 & 14.5 & 15 & 120 \\
 & 3 & 82.8 & 81.6 & 78.4 & 46.7 & 40.0 & 10.6 & 24.2 & 11.8 & 35 & 249 \\
 & 5 & 86.9 & 86.7 & 84.5 & 61.0 & 60.0 & 19.1 & 36.8 & 9.6 & 48 & 300 \\
 & 10 & 92.9 & 93.9 & 88.7 & 78.0 & 72.0 & 31.9 & 42.1 & 6.9 & 69 & 344 \\
 & 50 & -- & -- & -- & 81.3 & -- & -- & -- & 5.4 & 270 & 1084 \\
\midrule
\multirow{5}{*}{Qwen3 (untrained)} & 1 & 73.7 & 69.4 & 77.3 & 32.7 & 28.0 & 9.6 & 14.7 & 13.0 & 13 & 104 \\
 & 3 & 66.7 & 64.3 & 71.1 & 36.7 & 26.0 & 7.4 & 22.1 & 11.9 & 36 & 255 \\
 & 5 & 75.8 & 78.6 & 82.5 & 58.0 & 54.0 & 12.8 & 32.6 & 9.3 & 47 & 295 \\
 & 10 & 88.9 & 85.7 & 94.8 & 69.3 & 65.0 & 25.5 & 43.2 & 7.3 & 73 & 392 \\
 & 50 & -- & -- & -- & 78.7 & -- & -- & -- & 5.5 & 275 & 1163 \\
\midrule
Flash $+$ \router\ ($3$-hop) & chain & 85.9 & 62.2 & 58.8 & 59.7 & 57.0 & 12.8 & 6.3 & 2.6 & 14 & 42 \\
Flash $+$ \router\ (mixed-hop) & chain & 80.8 & 80.6 & 89.7 & 85.4 & 84.0 & 10.6 & 5.3 & 2.1 & 12 & 33 \\
\bottomrule
\end{tabular}
\end{table}

\clearpage

\subsection{Decision Penalty}
\label{sec:appendix:lenpen}

Figure~\ref{fig:lenpen_full} adds training reward and stopping behavior to
the rollout length and test accuracy shown in the main text.
Table~\ref{tab:lenpen} gives the checkpoint scores and prompt-transfer
results. At iteration $200$, the $\lambda=0$ and $\lambda=0.02$ runs reach
$48.0\%$ and $47.7\%$ on L4. On the vague prompts, the mean accuracy rises
from $26.0\%$ to $32.3\%$, although L4 decreases from $20.0\%$ to $18.0\%$.
The mean stop rate increases and the mean number of backtracks decreases.
Each setting has one run, so these results do not establish statistical
equivalence or seed-level reliability. Appendix~\ref{app:rl} records the
optimization settings, including the gradient-memory adjustments.

\begin{figure}[ht]
    \centering
    \includegraphics[width=\textwidth]{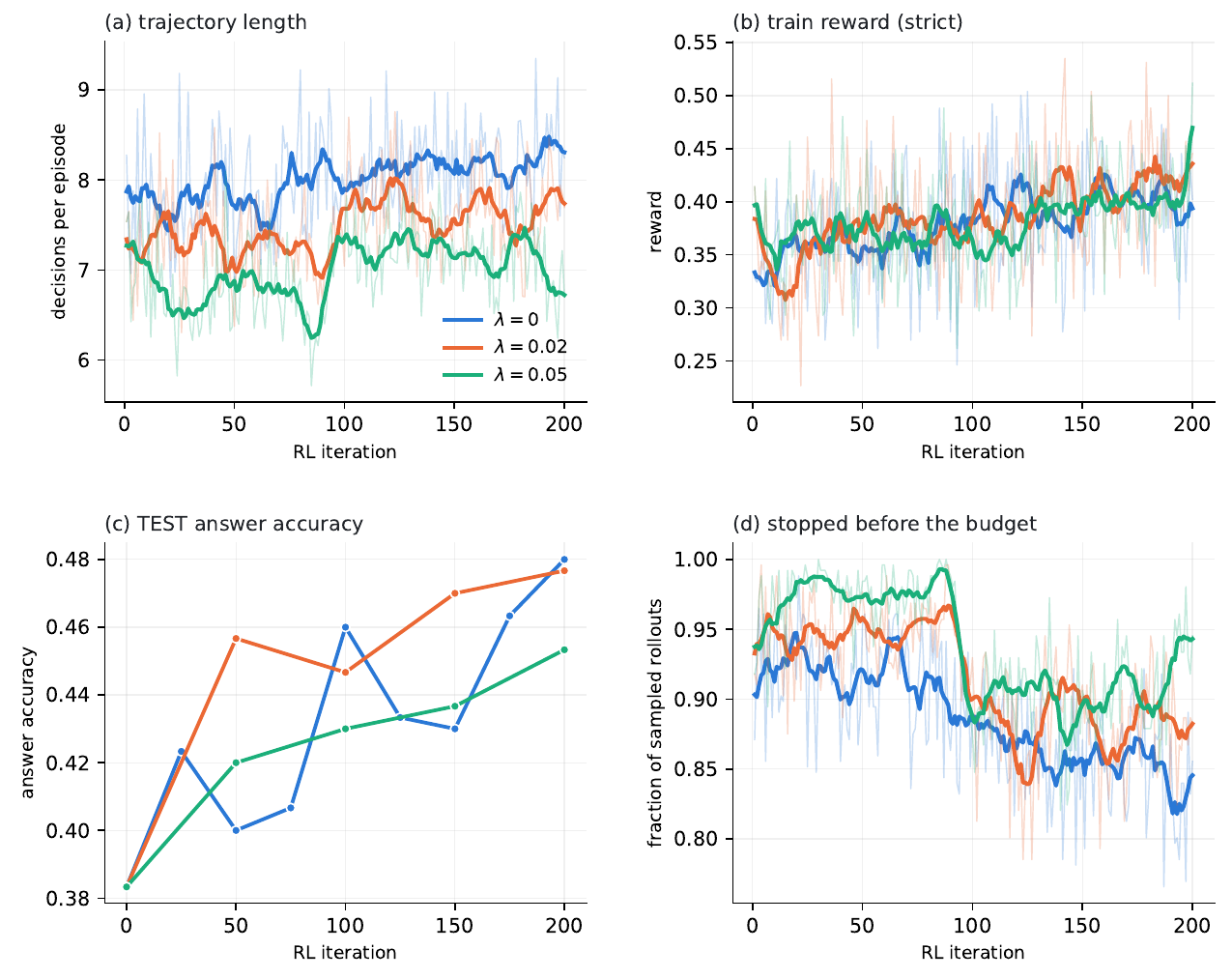}
    \caption{\textbf{Full training dynamics with a decision penalty.}
    Runs use \(\lambda\in\{0,0.02,0.05\}\) with the same
    \(16\)-decision budget.
    \textbf{(a)} Mean decisions per rollout.
    \textbf{(b)} Unpenalized terminal reward \(r\), defined identically across runs.
    \textbf{(c)} Success rate on the preregistered \(300\)-question L4 test set.
    \textbf{(d)} Fraction of sampled rollouts that issue \textsc{Stop} before
    exhausting the budget.
    Thin curves show per-iteration values and thick curves show moving
    averages in (a), (b), and (d).}
    \label{fig:lenpen_full}
\end{figure}
\begin{table}[H]
\centering
\caption{\textbf{Decision-penalty results.}
Comparison of $\lambda=0$ and $\lambda=0.02$ with gated decoding.
\emph{In-domain} rows use the L4 test set. \emph{Vague} rows measure
stopping on a verified acceptable answer, using development world $w_0$
except for L4, which uses the test worlds. \emph{Rephrased} rows use two
natural-language versions of the L4 development questions; they are
separate from the human-authored dataset. Stop rate and backtracks are
means over the five vague-prompt levels. Accuracy is in percent; stop rate
is a fraction. Bold marks the better value. Each setting uses one seed.}
\label{tab:lenpen}
\small
\setlength{\tabcolsep}{4pt}
\renewcommand{\arraystretch}{1.05}
\begin{tabular}{@{}llrrr@{}}
\toprule
& & $\lambda=0$ & $\lambda=0.02$ & $\Delta$ \\
\midrule
\multirow{3}{*}{\rotatebox[origin=c]{90}{\scriptsize in-dom.}} & L4 test, iter $50$ & 41.0 & \textbf{45.7} & $+4.7$ \\
 & L4 test, iter $100$ & \textbf{46.0} & 44.7 & $-1.3$ \\
 & L4 test, iter $200$ & \textbf{48.0} & 47.7 & $-0.3$ \\
\midrule
\multirow{6}{*}{\rotatebox[origin=c]{90}{\scriptsize vague}} & L1 & 32.3 & \textbf{37.4} & $+5.1$ \\
 & L2 & 34.7 & \textbf{49.0} & $+14.3$ \\
 & L3 & 27.1 & \textbf{31.2} & $+4.1$ \\
 & L4 & \textbf{20.0} & 18.0 & $-2.0$ \\
 & L5 & 16.0 & \textbf{26.0} & $+10.0$ \\
 & mean & 26.0 & \textbf{32.3} & $+6.3$ \\
\midrule
\multirow{2}{*}{\rotatebox[origin=c]{90}{\scriptsize repr.}} & rephrased v1 & 19.0 & 19.0 & $\pm0.0$ \\
 & rephrased v2 & 26.0 & \textbf{31.0} & $+5.0$ \\
\midrule
\multirow{2}{*}{\rotatebox[origin=c]{90}{\scriptsize beh.}} & stop rate & 0.43 & \textbf{0.51} & $+0.08$ \\
 & backtracks & 2.69 & \textbf{2.22} & $-0.47$ \\
\bottomrule
\end{tabular}
\end{table}

\section{Generalization and Limitations}
\label{sec:ana-robust}

\subsection{Three-Object Scenes}

The three-object evaluation adds a distractor to each two-object hop scene
while preserving the question and target. Query views and answer portraits
are unchanged. The added objects come from a separate pool. The standalone
\router{} rows in Figure~\ref{fig:three_objects} use the policy alone;
Flash appears only in the rows explicitly labeled as agents. No agent with
the trained chain tool is included in that comparison.

The precise three-hop and mixed-hop policies decrease from $48.0\%$ to
$30.4\%$ and from $76.3\%$ to $52.2\%$, respectively. The perturbed
results average the three test worlds, with $99$, $100$, and $100$ valid
queries. The drop shows sensitivity to competing visual content, but this
experiment alone does not separate failures of object recognition from
failures of search control.

\subsection{Four-Hop Extrapolation}
\label{sec:ana-four-hop}

The four-hop evaluation extends L4 beyond the maximum training depth.
It uses $97$ questions over a $600$-image development corpus. All policies
receive precise prompts. Table~\ref{tab:plus-one-hop} reports the better of
two decoding settings for each checkpoint: a question-specific minimum
stop depth and stopping permitted after the first retrieval. Thus the
comparison gives each policy the benefit of either stopping constraint.

\begin{table}[H]
\centering
\caption{\textbf{Four-hop answer accuracy (\%).} Standalone policies on
L4 four-hop development queries, with a $16$-decision budget. Values use
the better of the two stopping constraints described above.}
\label{tab:plus-one-hop}
\small
\setlength{\tabcolsep}{8pt}
\begin{tabular}{@{}lr@{}}
\toprule
Training configuration & Accuracy \\
\midrule
Precise, $3$-hop & 1.0 \\
Precise, mixed-hop & 1.0 \\
Vague, $3$-hop & 8.2 \\
Vague, mixed-hop & 1.0 \\
\bottomrule
\end{tabular}
\end{table}

All four policies remain at or below $8.2\%$. For the precise three-hop
policy, the depth floor gives $0.0\%$ and removing it gives $1.0\%$;
the precise mixed-hop policy obtains $1.0\%$ either way. The depth floor
therefore does not resolve the failure. This test also increases the corpus
from roughly $450$ to $600$ images and adds another layer of distractors.
It measures transfer to longer, harder searches, rather than isolating hop
count alone. Training on four-hop worlds is a possible next experiment;
recovery from such training has not been established here.

\subsection{Realistic Queries and Settings}
\label{app:human-authored}

Table~\ref{tab:natural-full} expands the results on realistic queries in
Table~\ref{tab:natural}. All methods use the same $50$ questions and
$386$-image corpus, without further training. Annotators write the queries
and specify everyday scenes; the images are re-rendered while preserving
the intended content. These are human-authored tasks with rendered images,
so the experiment measures transfer beyond procedural prompts and scene
templates, rather than performance on unmodified personal photographs.

\begin{table}[H]
    \centering
    \caption{\textbf{Full results on realistic queries.}
    Answer accuracy and set success (\%) for standalone \router{} and
    Flash with each retrieval tool. The metric definitions are given below.}
    \label{tab:natural-full}
    \small
    \setlength{\tabcolsep}{6pt}
    \renewcommand{\arraystretch}{1.1}
    \begin{tabular}{@{}llrr@{}}
        \toprule
        Method & Training configuration & Answer accuracy & Set success \\
        \midrule
        \multirow{4}{*}{\router}
            & Precise, \(3\)-hop  & 12.0 & 20.0 \\
            & Precise, mixed-hop & 4.0  & 20.0 \\
            & Vague, \(3\)-hop    & 4.0  & 4.0 \\
            & Vague, mixed-hop   & 14.0 & 14.0 \\
        \midrule
        \multirow{4}{*}{Flash \(+\) \router}
            & Precise, \(3\)-hop  & 20.0 & 36.0 \\
            & Precise, mixed-hop & 20.0 & 34.0 \\
            & Vague, \(3\)-hop    & 8.0  & 18.0 \\
            & Vague, mixed-hop   & 16.0 & 30.0 \\
        \midrule
        Flash \(+\) GE2   & Untrained & 12.0 & 22.0 \\
        Flash \(+\) Qwen3 & Untrained & 8.0  & 26.0 \\
        \bottomrule
    \end{tabular}
\end{table}

\paragraph{Metrics.}
For standalone \router{}, answer accuracy (\texttt{answer\_acc}) requires
\textsc{Stop} with an acceptable image at the top of the stack. Set success
(\texttt{e2e}) checks whether an acceptable image is anywhere in the final
active stack; popped images do not count. For Flash, answer accuracy
(\texttt{ans\_acc}) checks the final returned answer. If there is no valid
\texttt{FINAL} declaration, the evaluator uses the last answer candidate
returned by \router{}, or the last retrieved image for GE2 and Qwen3.
Set success (\texttt{final\_gt}) checks the cumulative returned image set
across tool calls. These set metrics therefore retain different histories
for standalone policies and agents.

\paragraph{Evaluation records and budgets.}
The precise three-hop and mixed-hop rows use
\texttt{rloo10w10\_200} and \texttt{mixh\_rl\_best}.
The archived human-authored runs label \texttt{vgo\_rl\_best} as vague
three-hop and \texttt{vgmix\_rl\_best} as vague mixed-hop. These are the
checkpoints underlying the vague rows here; they differ from the
\texttt{vg\_rl\_best}/\texttt{vgo\_rl\_best} pair used for the vague
blocks in the procedural main table.
The recorded Flash--\router{} runs allow up to eight chain calls, each
with at most $16$ policy decisions. GE2 and Qwen3 runs allow up to $16$
single-step tool calls. Each chain call may return multiple images, so a
chain call and a single-step call have different costs.

The main table reports the agent's set success: $36.0\%$ and $34.0\%$
for the two precise \router{} tools, compared with $22.0\%$ and $26.0\%$
for GE2 and Qwen3. Their answer accuracies are $20.0\%$, $20.0\%$,
$12.0\%$, and $8.0\%$. The gap between finding an answer and returning it
shows that answer selection remains a separate source of error. With only
$50$ queries, these results support a limited transfer claim, not reliable
performance across arbitrary photo collections.

\subsection{Extensions to Logo and Text Matching}
\label{sec:ana-limits}

The matching rule can extend beyond physical object identity. L6 adds
logo matching, and L7 adds printed-text matching. The rule is chosen
independently for each hop and specified in the instruction; the actual
logo or text must be read from the image. These extensions retain the
constraints and distractors from L1--L5. Rendering checks also cover logos
and printed text, and training and evaluation use disjoint logo and
printed-string pools.

The L6 evaluation world contains $94$ queries, $501$ corpus images, and
$50$ query images; L7 contains $95$ queries, $481$ corpus images, and
$50$ query images. Table~\ref{tab:matching-extensions} reports the
three-hop results with precise prompts. Tables~\ref{tab:main-h1}
and~\ref{tab:main-h2} include the one- and two-hop results. The trained
policies use the same L4 checkpoints as the main evaluation, without
further training on either extension.

\begin{table}[H]
\centering
\caption{\textbf{Three-hop success on logo and text matching (\%).}
These rows extend Table~\ref{tab:main}, with the same methods and scoring
rules. Both levels use precise prompts and development world $w_0$.
The strategy or training configuration is listed in the second column.}
\label{tab:matching-extensions}
\small
\setlength{\tabcolsep}{8pt}
\renewcommand{\arraystretch}{1.05}
\begin{tabular}{@{}llrr@{}}
\toprule
Method & Strategy / training & L6 (logo) & L7 (text) \\
\midrule
GE2 & Single-shot & 2.1 & 3.2 \\
GE2 & Greedy & 4.3 & 2.1 \\
GE2 & History & 2.1 & 2.1 \\
Qwen3 & Single-shot & 3.2 & 4.2 \\
Qwen3 & Greedy & 4.3 & 3.2 \\
Qwen3 & History & 4.3 & 3.2 \\
\midrule
Caption $\to$ GE2 & Single-shot & 2.1 & 0.0 \\
Caption $\to$ GE2 & Greedy & 1.1 & 0.0 \\
Caption $\to$ GE2 & History & 0.0 & 0.0 \\
Caption $\to$ Qwen3 & Single-shot & 3.2 & 4.2 \\
Caption $\to$ Qwen3 & Greedy & 3.2 & 2.1 \\
Caption $\to$ Qwen3 & History & 2.1 & 2.1 \\
\midrule
\router{} & Precise, $3$-hop & 7.4 & 1.1 \\
\router{} & Precise, mixed-hop & 5.3 & 2.1 \\
\midrule
Flash $+$ GE2 & Untrained & 6.4 & 12.6 \\
Flash $+$ Qwen3 & Untrained & 9.6 & 14.7 \\
Flash $+$ \router{} & Precise, $3$-hop & 12.8 & 6.3 \\
Flash $+$ \router{} & Precise, mixed-hop & 10.6 & 5.3 \\
\bottomrule
\end{tabular}
\end{table}

Training uses L4 identity links, so these results test transfer to matching
rules absent from post-training. With precise mixed-hop training,
standalone \router{} achieves $5.3\%$ on L6 and $2.1\%$ on L7;
Flash with the same tool achieves $10.6\%$ and $5.3\%$, respectively.
The results do not establish whether the main cause is reading the marks,
applying the new matching rule, or controlling the resulting search.
Wider retrieval improves some L6/L7 agent results in
Table~\ref{tab:topk-full}, but it does not resolve these questions.
Training with logo and text links, together with separate perception
checks, would be needed to distinguish these possibilities.

\clearpage
\section{Trajectory Case Studies}
\label{app:cases}

These examples show the retrieval decisions behind the results. In the
figures, \textsc{Push} denotes \textsc{Select} and \textsc{Pop} denotes
\textsc{Backtrack}; image identifiers refer to entries in the same world.

\subsection{Recovering from a Wrong Branch}

We compare checkpoints from precise three-hop training on L4 development
query \texttt{2\_leftof}. The solver must find object Q to the right of a
yellow object A, then A to the left of a green object B, and finally B alone.
The gold chain is \texttt{00076}$\to$\texttt{00079}$\to$\texttt{00352}.
Figure~\ref{fig:case_a_rl} contrasts Online~IL and RLVR after the same
initial mistake; the following figures show the other methods on this query.

\begin{figure}[H]
\centering
\includegraphics[width=\linewidth]{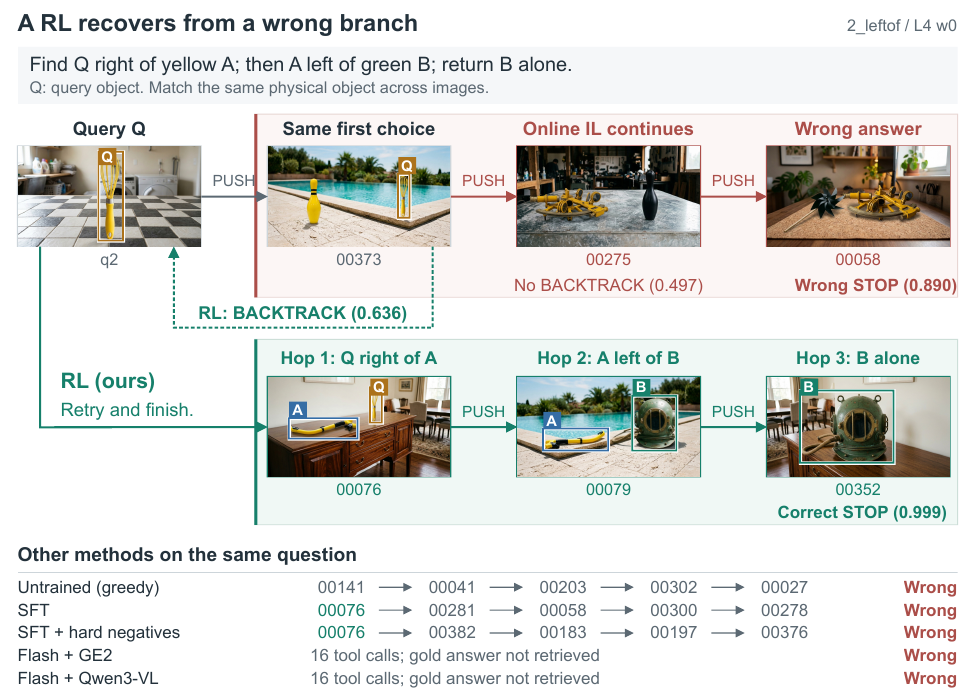}
\caption{\textbf{Recovery after the same incorrect first choice.}
Online~IL and RLVR both first select \texttt{00373}. Online~IL continues
along that branch and stops on a wrong answer. RLVR backtracks with
$p_{\mathrm{bt}}=0.636$, retrieves the three gold images, and stops with
$p_{\mathrm{stop}}=0.999$. Solid arrows show selections and the dashed
return shows backtracking. The lower rows summarize the other methods;
the untrained and SFT trajectories use five fixed retrieval steps.}
\label{fig:case_a_rl}
\end{figure}

\clearpage
\begin{figure}[H]
\centering
\includegraphics[width=\linewidth]{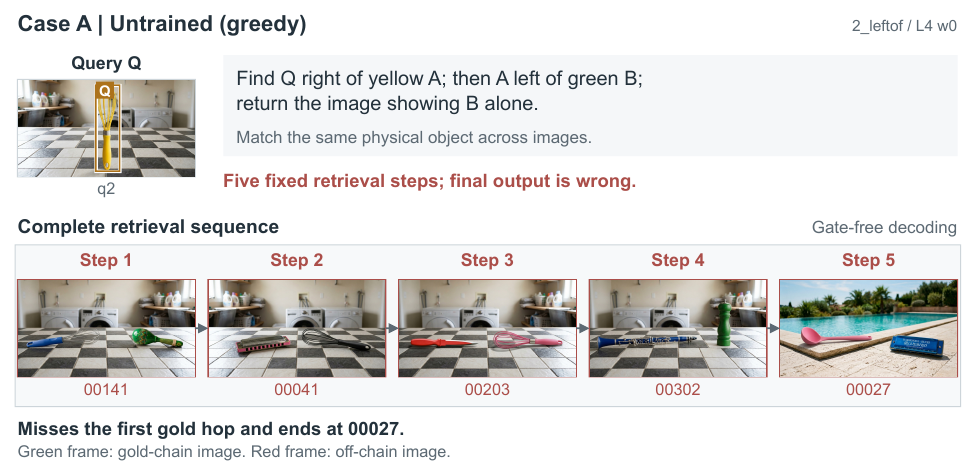}
\caption{\textbf{Untrained retrieval misses the first hop.}
The complete five-step greedy trajectory for \texttt{2\_leftof} ends at
\texttt{00027}. Every retrieved image lies outside the gold chain (red
frames). This decoder takes five steps without learned backtracking or
stopping.}
\label{fig:case_a_untrained}
\end{figure}

\begin{figure}[H]
\centering
\includegraphics[width=\linewidth]{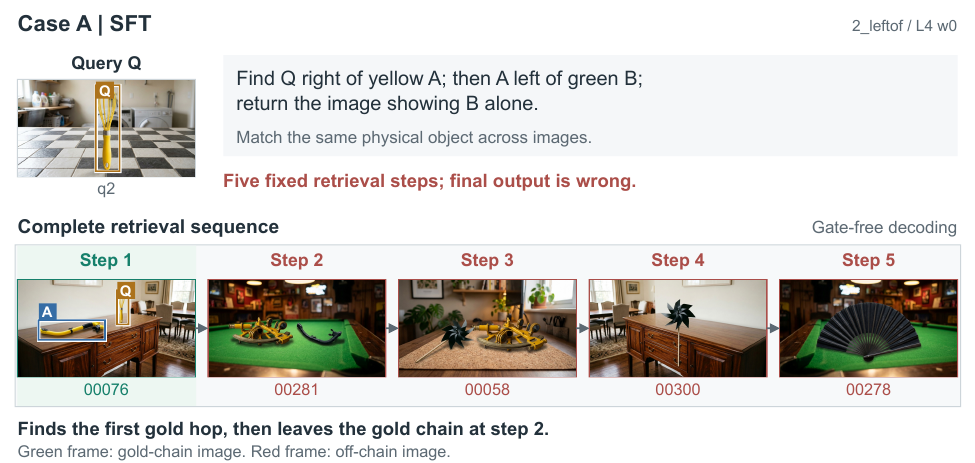}
\caption{\textbf{SFT retrieves the first hop, then leaves the chain.}
The state-encoder warm-up retrieves \texttt{00076} (green frame), then
selects \texttt{00281} and eventually ends at \texttt{00278}. The full
five-step trajectory is shown with the same decoder as the untrained model.}
\label{fig:case_a_sft}
\end{figure}

\clearpage
\begin{figure}[H]
\centering
\includegraphics[width=\linewidth,height=0.79\textheight,keepaspectratio]
    {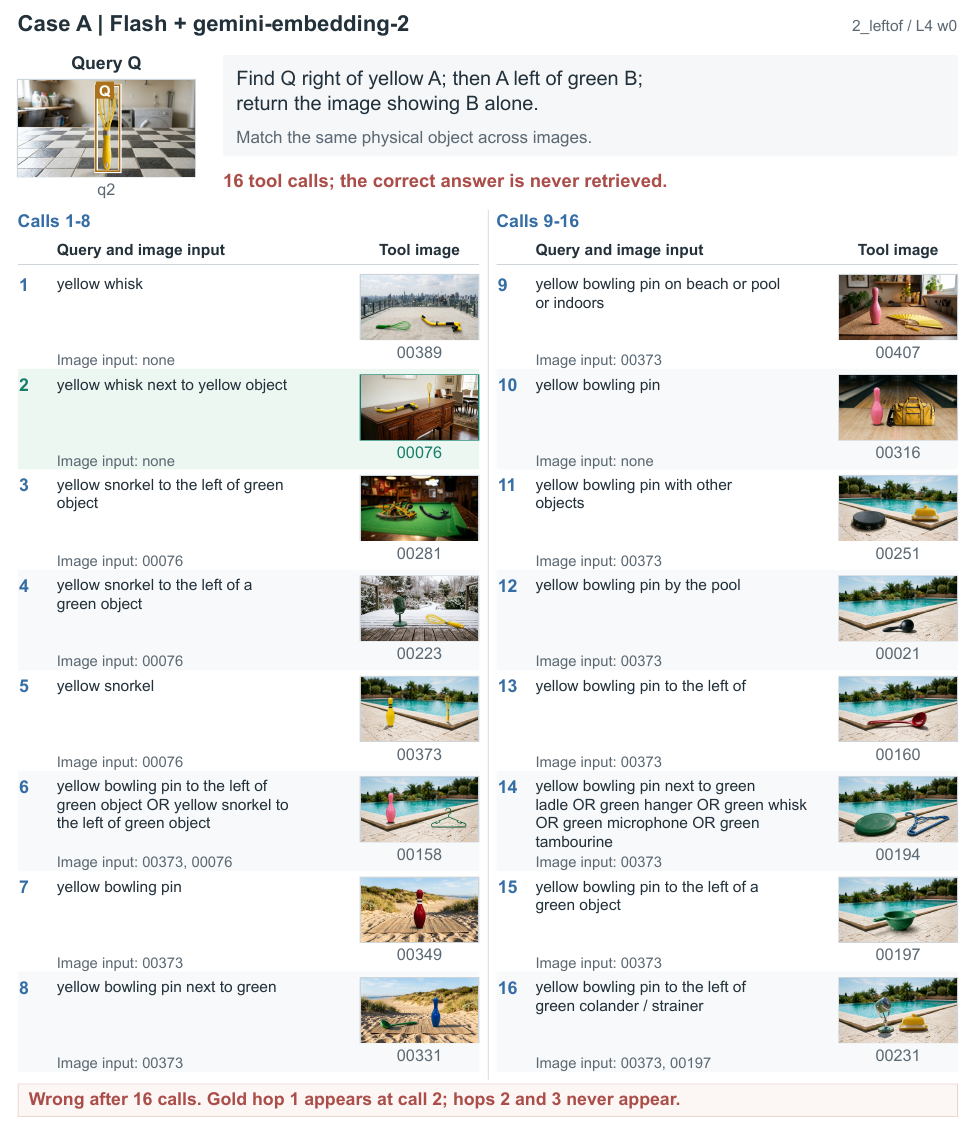}
\caption{\textbf{Flash with GE2 on the same query.}
All $16$ tool calls are shown, read down the left column and then down
the right. Each row gives the recorded query, image inputs, and retrieved
image. Call $2$ retrieves the first gold hop, but neither the second hop
nor the answer is retrieved. The recorded outcome is incorrect.}
\label{fig:case_a_agent_ge2}
\end{figure}

\clearpage
\begin{figure}[H]
\centering
\includegraphics[width=\linewidth,height=0.79\textheight,keepaspectratio]
    {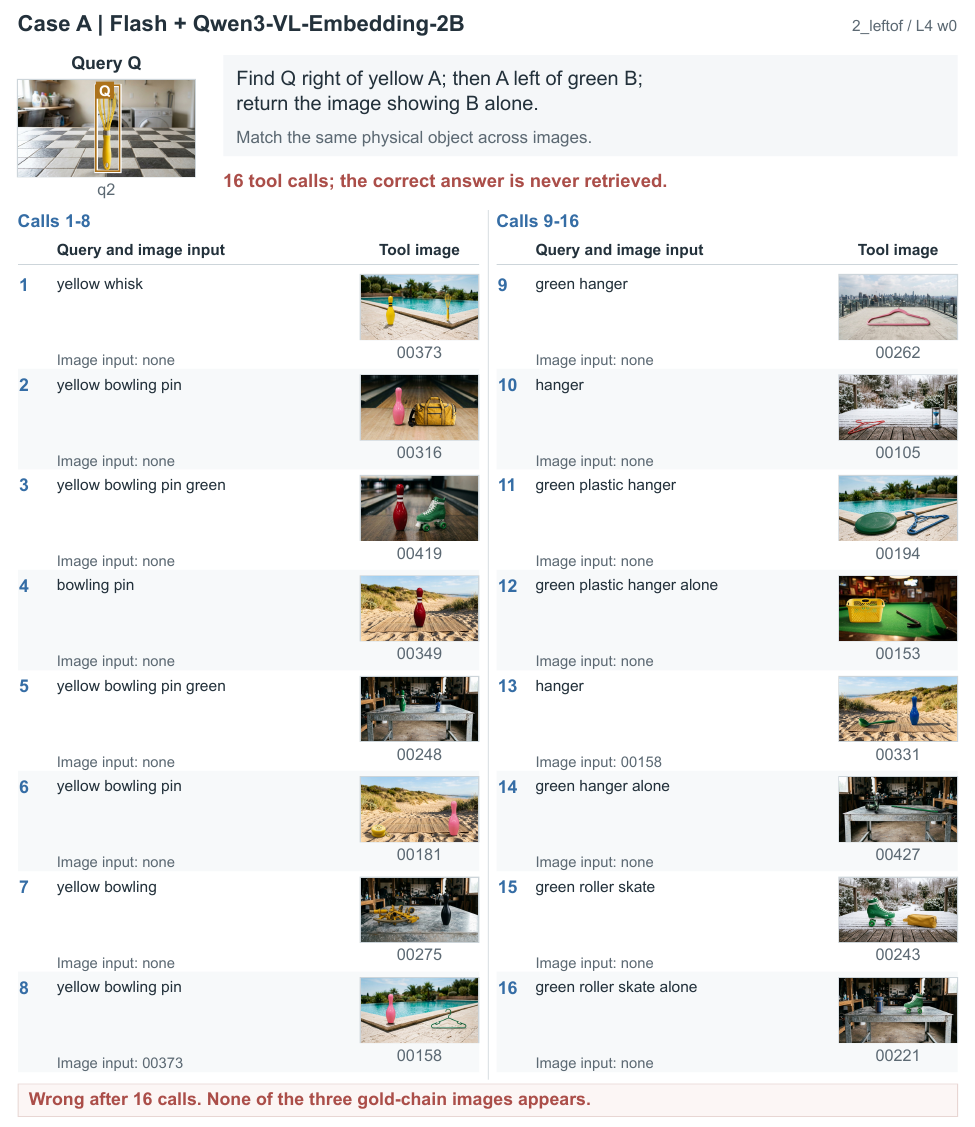}
\caption{\textbf{Flash with Qwen3 on the same query.}
The complete $16$-call trace uses the same layout as
Figure~\ref{fig:case_a_agent_ge2}. None of the three gold images is
retrieved. The last thumbnail shows the final tool result; the figure
does not supply a separate answer declaration.}
\label{fig:case_a_agent_qwen3}
\end{figure}

\clearpage
\subsection{Revising the Query and Selecting an Answer}

An agent can help at two different points. If a returned chain lacks the
answer, it can call the policy again with revised text. If the answer is
already present, it can select that image even when the policy continues
searching. Figures~\ref{fig:case_b_agent_requery} and~\ref{fig:answer-selection}
show these behaviors on two separate queries.

\begin{figure}[H]
\centering
\includegraphics[width=\linewidth]{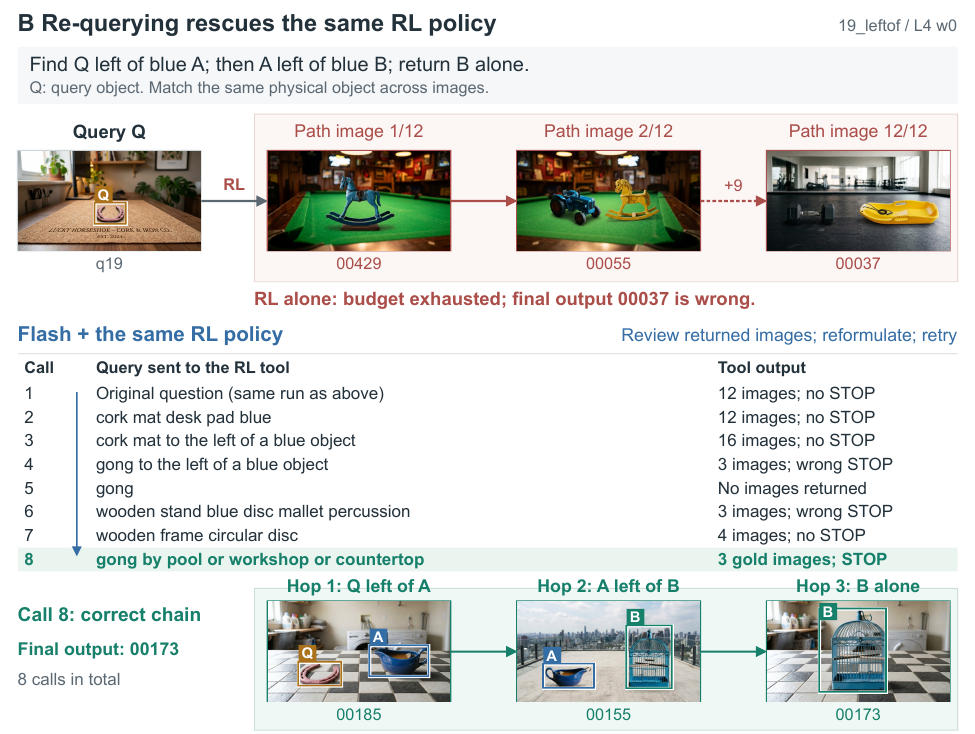}
\caption{\textbf{Re-querying the same policy recovers a valid chain.}
On L4 development query \texttt{19\_leftof}, the initial policy call
exhausts its budget and returns a chain without the answer. Flash retries
the same checkpoint and query image with revised text. The eighth call
retrieves \texttt{00185}$\to$\texttt{00155}$\to$\texttt{00173} and
stops correctly. Query strings are copied from the trace. The recorded
system answer is \texttt{00173}; the trace does not contain a separate
\texttt{FINAL} declaration.}
\label{fig:case_b_agent_requery}
\end{figure}

\begin{figure}[H]
\centering
\includegraphics[width=\linewidth]{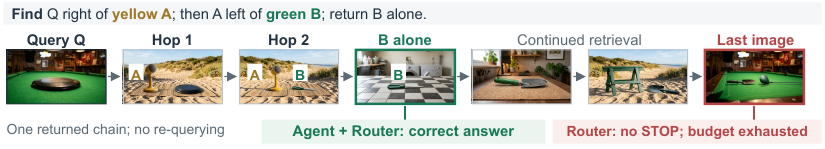}
\caption{\textbf{Answer selection from the same returned chain.}
On L4 development query \texttt{33\_leftof}, precise mixed-hop \router{}
retrieves the gold chain but continues searching until its $16$-action
budget runs out. Its returned stack contains six images and ends at the
incorrect image \texttt{00435}. Without re-querying, Flash selects the
correct third image, \texttt{00247}.}
\label{fig:answer-selection}
\label{fig:case-b}
\end{figure}

\clearpage
\subsection{A Vague Query with Two Acceptable Answers}

The vague question in Figure~\ref{fig:vague-branches} omits the relation
that distinguishes two precise chains. Both share the first hop, and both
endpoints are verified answers. The demonstration record specifies one
branch; the outcome reward accepts either endpoint. In this trace, RLVR
first explores the demonstrated branch, backtracks, and stops at the
other acceptable endpoint. Online~IL repeatedly backtracks and exhausts
its budget without returning an answer.

\begin{figure}[H]
\centering
\includegraphics[width=\linewidth,height=0.83\textheight,keepaspectratio]
    {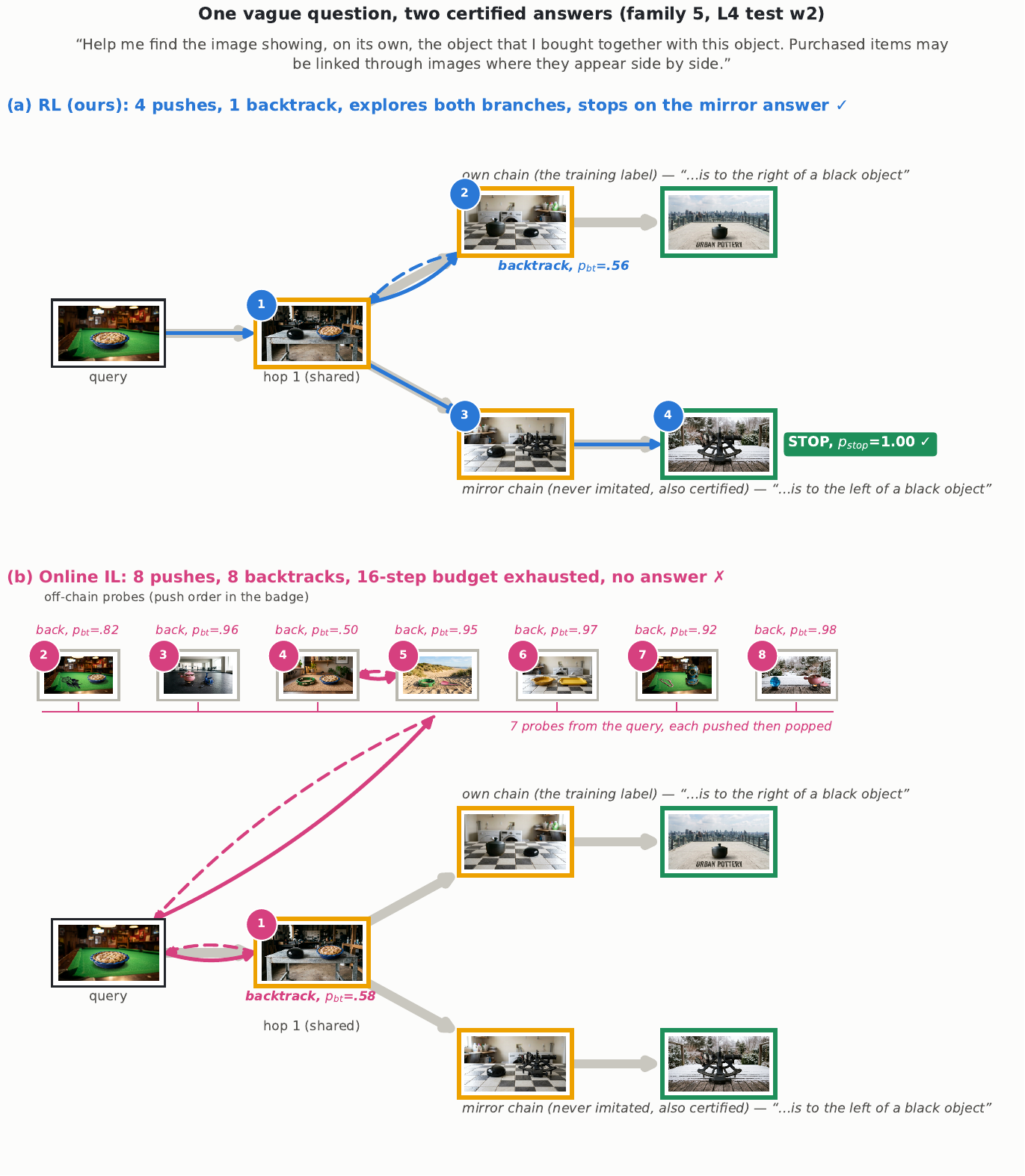}
\caption{\textbf{Two valid endpoints under one vague instruction.}
Family $5$ from L4 test world $w_2$. Gray paths show the two verified
chains; colored arrows show the recorded policy decisions. Solid arrows
are selections and dashed arrows are backtracks. RLVR makes four
selections, backtracks once, and stops at the alternative valid answer.
Online~IL makes eight selections and eight backtracks, exhausting its
$16$-decision budget. Both are standalone policies.}
\label{fig:vague-branches}
\end{figure}

\end{document}